\def\alt{0} 

\ifnum\alt=1
\documentclass[twoside,11pt]{article}
\usepackage{jmlr2e}

\usepackage{soul}
\usepackage[hang,flushmargin]{footmisc}
\usepackage{scrextend}

\usepackage[nottoc,notlot,notlof]{tocbibind}

\hypersetup{
	colorlinks,
	citecolor=blue,
	filecolor=blue,
	linkcolor=blue,
	urlcolor=blue
}

\else
\documentclass[a4paper]{article}
\usepackage[a4paper,top=3.5cm,bottom=3cm,left=3.5cm,right=3.5cm,marginparwidth=1.5cm]{geometry}
\usepackage{amsthm}
\fi

\usepackage[utf8]{inputenc}
\usepackage{bbm}
\usepackage{times}
\usepackage{graphicx,color}
\usepackage{array,float}
\usepackage{url}

\usepackage{amstext,amssymb,amsmath}
\usepackage{hyphenat}
\usepackage{verbatim}
\usepackage{bm}
\usepackage{paralist}
\usepackage{ulem}
\usepackage{todonotes}
\usepackage{paralist}\usepackage{bbm}
\usepackage{wrapfig}
\usepackage{hyperref}
\usepackage[noend]{algorithmic}
\usepackage{algorithm}
\usepackage{xparse,etoolbox}
\usepackage{threeparttable}
\usepackage{bbm}
\usepackage{enumitem}
\usepackage{dsfont}

\ifnum\alt=1
\fi

\ifnum\alt=0
\newtheorem{lem}{Lemma}[section]
\newtheorem{theorem}{Theorem}
\newtheorem{remark}[theorem]{Remark}
\newtheorem{thm}[lem]{Theorem}
\newtheorem{infthm}[lem]{Informal Theorem}
\newtheorem{cor}[lem]{Corollary}
\newtheorem{problem}[lem]{Problem}
\newtheorem{defn}[lem]{Definition}

\newtheorem{fact}[lem]{Fact}
\newtheorem{assumption}[lem]{Assumption}
\newtheorem{claim}[lem]{Claim}

\else
\newtheorem{lem}[theorem]{Lemma} 
\newtheorem{thm}[theorem]{Theorem}

\newtheorem{defn}[theorem]{Definition}

\newtheorem{claim}[theorem]{Claim}

\fi

\renewcommand{\paragraph}[1]{\vspace{3pt}\noindent\textbf{#1}}

\newcommand{\cO}{\ensuremath{\mathcal{O}}}
\newcommand{\cX}{\ensuremath{\mathcal{X}}}
\newcommand{\cY}{\ensuremath{\mathcal{Y}}}

\newcommand{\hS}{\widehat S}

\newcommand{\hdel}{\hat{\delta}}
\newcommand{\heps}{\hat{\epsilon}}

\newcommand{\hy}{\hat{y}}
\newcommand{\hx}{\hat{x}}
\newcommand{\hh}{\hat{h}}
\newcommand{\tlh}{\tilde{h}}

\newcommand{\relb}{\mathsf{Relabel}}

\newcommand{\tx}{\tilde{x}}
\newcommand{\ty}{\tilde{y}}

\newcommand{\Tpub}{T_\mathsf{pub}}
\newcommand{\hatm}{\hat{m}}
\newcommand{\tT}{\tilde{T}}
\newcommand{\tH}{\tilde{\mathcal{H}}}

\newcommand{\pac}{\mathsf{PAC}}
\newcommand{\prv}{\mathsf{priv}}

\newcommand{\cGTri}{\mathcal{G}_{\cH_\Delta}}

\newcommand{\prvcl}{\mathsf{PrivClass}}
\newcommand{\Agprv}{\mathsf{AgPrivCl}}
\newcommand{\uniprv}{\mathsf{UnvPrivCl}}

\newcommand{\hpv}{h_{\mathsf{priv}}}

\newcommand{\erm}{\mathsf{ERM}}

\newcommand{\vc}{\mathsf{VC}}

\newcommand{\pcqr}{\mathsf{PCQR}}

\newcommand{\stab}{\mathsf{stab}}

\DeclareMathOperator*{\argmin}{arg\,min}

\newcommand{\pr}[2]{\underset{#1}{\mathbb{P}}\left[ #2 \right]}
\newcommand{\ex}[2]{\underset{#1}{\mathbb{E}}\left[ #2 \right]}

\def\polylog{\operatorname{polylog}}

\newcommand{\eps}{\epsilon}

\newcommand{\cA}{\mathcal{A}}

\newcommand{\cD}{\mathcal{D}}

\newcommand{\cG}{\mathcal{G}}
\newcommand{\cH}{\mathcal{H}}

\newcommand{\ind}{{\mathbf{1}}}

\newcommand{\cF}{\mathcal{F}}

\newcommand{\dis}{\mathsf{dis}}
\newcommand{\samp}{\mathsf{SubSamp}}

\newcommand{\pub}{\mathsf{pub}}

\newcommand{\hw}{\hat{w}}

\newcommand{\R}{\mathcal{R}}

\newcommand{\cB}{\mathcal{B}}

\newcommand{\cQ}{\mathcal{Q}}

\newcommand{\ignore}[1]{}

\newcommand{\cR}{\mathcal{R}}
\newcommand{\A}{\mathcal{A}}

\newcommand{\bc}{\mathbf{c}}

\newcommand{\whh}{\widehat{h}}
\newcommand{\wtH}{\widetilde{\mathcal{H}}}

\newcommand{\err}{\mathsf{err}}
\newcommand{\herr}{\widehat{\mathsf{err}}}

\newcommand{\dist}{{\sf dist}}
\newcommand{\sspp}{{\sf SSPP}}
\newcommand{\thr}{\Gamma}
\newcommand{\ds}{{\sf dist}\,}
\newcommand{\wds}{\widehat{\ds}}

\begin{document}
\ifnum\alt=1


\ShortHeadings{Privately Answering Classification Queries in the Agnostic PAC Model}{Nandi and Bassily}
\firstpageno{1}

	\title{Privately Answering Classification Queries in the Agnostic PAC Model}
	
	\author{\name Anupama Nandi \thanks{Part of this work was done while the authors were visiting Simons Institute for the Theory of Computing. Research supported by NSF Awards AF-1908281, SHF-1907715, Google Faculty Research Award, and OSU faculty start-up support.} \email nandi.10@osu.edu \\
		\addr Department of Computer Science \& Engineering\\
		The Ohio State University
		\AND
		\name  Raef Bassily \footnotemark[1]  \email bassily.1@osu.edu \\
		\addr Department of Computer Science \& Engineering\\
		The Ohio State University}
	
	\editor{}
\else

\title{Privately Answering Classification Queries in the Agnostic PAC Model}
\author{Anupama Nandi \thanks{Part of this work was done while the authors were visiting Simons Institute for the Theory of Computing. Research supported by NSF Awards AF-1908281, SHF-1907715, Google Faculty Research Award, and OSU faculty start-up support.} \qquad \qquad  Raef Bassily \footnotemark[1] 
\\
\\ Department of Computer Science \& Engineering \\The Ohio State University}
\date{}

\fi

	\maketitle

\begin{abstract}%
    We revisit the problem of differentially private release of classification queries. In this problem, the goal is to design an algorithm that can accurately answer a sequence of classification queries based on a private training set while ensuring differential privacy. 
    We formally study this problem in the  agnostic PAC model and derive a new upper bound on the private sample complexity. Our results improve over those obtained in a recent work \ifthenelse{\alt=1}{\citep{bassily2018model}}{\cite{bassily2018model}} for the agnostic PAC setting. In particular, we give an improved construction that yields a tighter upper bound on the sample complexity. Moreover, unlike \ifthenelse{\alt=1}{\citep{bassily2018model}}{\cite{bassily2018model}}, our accuracy guarantee does not involve any blow-up in the approximation error associated with the given hypothesis class. 
    
    Given any hypothesis class with VC-dimension $d$, we show that our construction can privately answer up to $m$ classification queries with average excess error $\alpha$ using a private sample of size $\approx \frac{d}{\alpha^2}\,\max\left(1, \sqrt{m}\,\alpha^{3/2}\right)$ (assuming the privacy parameter $\eps=\Theta(1)$). Using recent results on private learning with auxiliary public data, we extend our construction to show that one can privately answer \emph{any} number of classification queries with average excess error $\alpha$ using a private sample of size $\approx \frac{d}{\alpha^2}\,\max\left(1, \sqrt{d}\,\alpha\right)$. When $\alpha=O\left(\frac{1}{\sqrt{d}}\right)$ and the privacy parameter $\eps=\Theta(1)$, our private sample complexity bound is essentially optimal. 

\end{abstract}
\ifnum\alt=1
\fi

\section{Introduction}\label{sec:intro}


In this paper, we revisit the problem of answering a sequence of classification queries in the \emph{agnostic} PAC model under the constraint of $(\eps, \delta)$-differential privacy. An algorithm for this problem is given a \emph{private} training dataset $S=\{(x_1, y_1), \ldots, (x_n, y_n)\}$ of $n$ i.i.d. binary-labeled examples drawn from some unknown distribution $\cD$ over $\cX\times \cY$, where $\cX$ denotes an arbitrary data domain (space of feature-vectors) and $\cY$ denotes a set of binary labels (e.g., $\{0, 1\}$). The algorithm is also given as input some hypothesis class $\cH\subseteq \{0, 1\}^{\cX}$ of binary functions mapping $\cX$ to $\cY$. The algorithm accepts a sequence of classification queries given by a sequence of i.i.d. feature-vectors $\cQ=(\tx_1, \tx_2, \ldots)$, drawn from the marginal distribution of $\cD$ over $\cX$, denoted as $\cD_{\cX}$. Here, the feature-vectors defining the set of queries $\cQ$ do not involve any privacy constraint. The queries are also assumed to arrive one at a time, and the algorithm is required to answer the current query $\tx_j$ by predicting a label $\hy_j$ for it before seeing the next query (\emph{online setting}). The goal is to answer up to a given number $m$ of queries (which is a parameter of the problem) such that, (i) the entire process of answering the $m$ queries is $(\eps, \delta)$-differentially private, and (ii) the average excess error in the predicted labels does not exceed some desired level $\alpha \in (0, 1)$; specifically, $\frac{1}{m}\sum_{j=1}^m \ind\left(\hy_j\neq \ty_j\right)\leq \alpha + \min\limits_{h\in\cH}\err\left(h; \cD\right),$ where $\ty$ is the corresponding (hidden) true label, and $\min\limits_{h \in \cH}\err(h; \cD)$ is the approximation error associated with $\cH$, i.e., the least possible true (population) error that can be attained by a hypothesis in $\cH$ (see Section~\ref{sec:prelim} for formal definitions).

One could argue that a more direct approach for differentially private classification would be to design a differentially private learner that, given a private training set as input, outputs a classifier that is safe to publish and then can be used to answer any number of classification queries. However, there are several pessimistic results that either limit or eliminate the possibility of differentially private learning even for elementary problems such as one-dimensional thresholds \ifthenelse{\alt=1}{\citep{bun2015differentially, almm19}}{\cite{bun2015differentially, almm19}}. Therefore, it is natural to study the problem of classification-query release under differential privacy as an alternative approach.

A recent formal investigation of this problem was carried out in \ifthenelse{\alt=1}{\citep{bassily2018model}}{\cite{bassily2018model}}. This recent work gives an algorithm based on a combination of two useful techniques from the literature on differential privacy, namely, the \emph{sub-sample-and-aggregate} technique \ifthenelse{\alt=1}{\citep{NRS07, ST13}}{\cite{NRS07, ST13}} and the \emph{sparse-vector} technique \ifthenelse{\alt=1}{\citep{DR14}}{\cite{DR14}}. \ifthenelse{\alt=1}{The algorithm by}{The algorithm in} \cite{bassily2018model}, hereafter denoted as $\cA_{\samp}$, assumes oracle access to a generic, non-private (agnostic) PAC learner $\cB$ for $\cH$. In this work, we give non-trivial improvements over the results of \cite{bassily2018model} in the agnostic PAC setting. More details on the comparison with \cite{bassily2018model} are given in the ``Related work'' section below. Our improvements are in terms of the attainable accuracy guarantees and the associated private sample complexity bounds in the agnostic setting. These improvements are achieved via importing new ideas and techniques from literature (particularly, the elegant agnostic-to-realizable reduction technique of \ifthenelse{\alt=1}{\citep{BeimelNS14}}{\cite{BeimelNS14}}) to provide an improved construction 
for the one that appeared in \ifthenelse{\alt=1}{\citep{bassily2018model}}{\cite{bassily2018model}}.

\subsection*{Main results}
In this work, we formally study algorithms for classification queries release under differential privacy in the agnostic PAC model. We focus on the sample complexity of such algorithms as a function of the privacy and accuracy parameters as well as the number of queries to be answered. For simplicity, in the expressions given below for our upper bounds, we will assume that $\eps=\Theta(1)$.
\begin{itemize}[leftmargin=*]
    \item We give an algorithm for this problem that is well-suited for the agnostic setting. Our algorithm is a two-stage construction that is based on a careful combination of the relabeling technique of \ifthenelse{\alt=1}{\citep{BeimelNS14}}{\cite{BeimelNS14}} and the private classification algorithm $\cA_{\samp}$ \ifthenelse{\alt=1}{by}{of} \cite{bassily2018model} (see ``Techniques'' section below). 
    
    \item We show that our construction provides significant improvements over the results of \cite{bassily2018model} for the agnostic setting:
    \begin{itemize}[leftmargin=*]
        \item The error guarantees in \ifthenelse{\alt=1}{\citep{bassily2018model}}{\cite{bassily2018model}} involves a constant blow-up (a multiplicative factor $\approx 3$) in the approximation error $\min\limits_{h\in\cH}\err(h; \cD)$ associated with the given hypothesis class $\cH$. Using our construction, we give a standard excess error guarantee that does not involve such a blow-up.
        \item We show that our construction can answer up to $m$ queries with average excess error $\alpha$ using a private sample whose size $\approx \vc(\cH)/\alpha^2\cdot \max\left(1, \sqrt{m}\,\alpha^{3/2}\right)$ (assuming $\epsilon$ is a constant, e.g. $0.1$), where $\vc(\cH)$ is the VC-dimension of $\cH$. Note that this implies that we can answer up to $\approx 1/\alpha^3$ queries with private sample size that is essentially the same as the standard non-private sample complexity in the agnostic PAC model. i.e., that many queries can be answered with essentially no additional cost due to privacy. 
        \item Using \ifthenelse{\alt=1}{a recent result}{recent results} of \cite{ABM19} on the sample complexity of \emph{semi-private learners} (introduced \ifthenelse{\alt=1}{by}{in} \cite{beimel2013private}), we show that our construction immediately leads to a \emph{universal} private classification algorithm that can answer \emph{any} number of classification queries using a private sample of size $\approx \frac{\vc(\cH)}{\alpha^2}\cdot\max\left(1, \sqrt{\vc(\cH)}\,\alpha\right)$, which is independent of the number of queries. We note that when $\alpha = O\left(1/\sqrt{\vc(\cH)}\right)$ and assuming the privacy parameter $\eps=\Theta(1),$ our sample bound nearly matches the standard non-private sample complexity in agnostic PAC model. This implies that in this regime, we attain a nearly optimal sample complexity bound for privately answering \emph{any} number of classification queries. Equivalently, our bound is nearly optimal for any class $\cH$ with $\vc(\cH)=O(1/\alpha^2)$. We note that the setting studied \ifthenelse{\alt=1}{by}{in} \cite{ABM19} is tantamount to the setting of offline (batch) classification where the whole set of unlabeled data (the set of queries in our case) is available and given to the algorithm beforehand. Whereas, as described earlier, in this work we study the online setting (which was also studied \ifthenelse{\alt=1}{by}{in} \cite{bassily2018model}). Hence, the upper bound on the private sample complexity obtained \ifthenelse{\alt=1}{by}{in} \cite{ABM19} is \emph{not} valid in our setting.  
        

        
    \end{itemize}

\end{itemize}


\paragraph{Techniques: } Our algorithm is a two-stage construction. In the first stage, the input training set is pre-processed once and for all via a relabeling procedure due to \ifthenelse{\alt=0}{Beimel et al.}{}\cite{BeimelNS14} in which the labels are replaced with the labels generated by an appropriately chosen hypothesis in the given hypothesis class $\cH$. This step allows us to reduce the agnostic setting to a realizable one. In the second stage, we first sample a new training set from the empirical distribution of the relabeled set in the first stage, then feed it to $\cA_{\samp}$ of \cite{bassily2018model} together with other appropriately chosen input parameters. To formally prove the accuracy guarantee of our construction, in our analysis we use some tools from learning theory (e.g., the uniform-convergence argument of Claim~\ref{clm:disRate}). As mentioned earlier, we also use the framework of semi-private learning \ifthenelse{\alt=1}{\citep{beimel2013private, ABM19}}{\cite{beimel2013private,ABM19}} to transform our algorithm into a universal private classification algorithm. 



\subsection*{Related work} 

Our results are most closely related to \ifthenelse{\alt=1}{the results given by}{} \cite{bassily2018model}. \ifthenelse{\alt=1}{They provide formal accuracy guarantees for their algorithm in both the realizable and agnostic settings of the PAC model.}{In \cite{bassily2018model}, Bassily et al. provide formal accuracy guarantees for their algorithm in both the realizable and agnostic settings of the PAC model.} However, the accuracy guarantees \cite{bassily2018model} provide for the agnostic setting is far from optimal. In particular, their guarantees involves a constant blow-up in the approximation error $\min\limits_{h\in\cH}\err(h; \cD)$, which would limit the utility of their construction in scenarios where the approximation error is not negligible. In fact, in most typical scenarios in practice, the approximation error associated with the hypothesis (model) class is a non-negligible constant, (e.g., the test error attained by some state-of-the-art neural networks on benchmark datasets can be as large as  $5\%$, or $10\%$). Our improved construction avoids this  blow-up in the approximation error.

The construction \ifthenelse{\alt=1}{by}{in} \cite{bassily2018model} can answer up to $m$ queries with average excess error $\alpha+O(\gamma)$ (where $\gamma = \min\limits_{h\in\cH}\err(h; \cD)$ is the approximation error) using a private sample of size $\approx \frac{\vc(\cH)}{\alpha^2}\cdot\max\left(1, \sqrt{m\,\alpha}\right)$ \ifthenelse{\alt=1}{\citep[follows from Theorem 3.5][]{bassily2018model}}{(follows from \cite[Theorem 3.5]{bassily2018model})}. Given our results discussed in the ``Main results'' section above, it follows that our sample complexity bound is tighter than that of \cite{bassily2018model} by roughly a factor of 
$\max\left(1, \min\left(\sqrt{m\,\alpha},~ 1/\alpha\right)\right)$. In particular, our bound is tighter by roughly a factor of $\sqrt{m\alpha}$ for $\frac{1}{\alpha} \leq m < \frac{1}{\alpha^3} $, and it is tighter by roughly a factor of $\frac{1}{\alpha}$ for $m \geq \frac{1}{\alpha^3}$. Equivalently, for the same private sample size, our construction can answer roughly a factor of $1/\alpha^2$ more queries than that of \cite{bassily2018model}. 


\ifthenelse{\alt=1}{\cite{bassily2018model}}{Bassily et al.~\cite{bassily2018model}} also extend their construction to provide a semi-private learner that can finally produce a classifier. This is done by answering a sufficiently large number of queries then applying the \emph{knowledge transfer} technique using the new training set formed by the set of answered queries. The output classifier can then be used to answer any subsequent queries, and hence, their extended construction provides a universal private classification algorithm. Their private sample complexity bound for this task is $\approx \vc(\cH)^{3/2}/\alpha^{5/2}$ \ifthenelse{\alt=1}{\cite[see Theorem~4.3][]{bassily2018model}}{(see \cite[Theorem~4.3]{bassily2018model})}. On the other hand, our universal private classification algorithm yields a private sample complexity bound $\approx \frac{\vc(\cH)}{\alpha^2}\cdot\max\left(1, \sqrt{\vc(\cH)}\,\alpha\right)$, which is tighter than that of \cite{bassily2018model} by roughly a factor of $\min\left(\sqrt{\frac{\vc(\cH)}{\alpha}},~\frac{1}{\alpha^{3/2}}\right)$. Moreover, our bound is nearly optimal when $\alpha=O\left(1/\sqrt{\vc(\cH)}\right)$.


\medskip

\paragraph{Other related works:~}\ifnum\alt=1 \cite{dwork2018privacy} consider the same problem, but focus on the sample complexity of a single query. In the agnostic PAC setting, their accuracy guarantee does not suffer from the constant blow-up in the approximation error as in the results of \cite{bassily2018model}. However, their sample complexity upper bound scales as $\approx \vc(\cH)/\eps\,\alpha^3$, which is sub-optimal.
Assuming $\epsilon=\Theta(1),$ our bound in the single-query setting (i.e., $m=1$) is essentially optimal as it nearly matches the standard non-private sample complexity in the agnostic PAC model. In an independent work, \cite{Dagan-Feldman2019} further study the connections between uniform stability and differential privacy in the context of PAC learning, and give a new algorithm that yields a sample complexity bound of $\approx \vc(\cH)/\alpha^2 + \vc(\cH)^2/\eps\,\alpha$ in the single-query setting. Their bound exhibits the optimal dependence on $\epsilon$, and when $\vc(\cH)\leq \eps/\alpha,$ it is tighter than our bound by a factor of $1/\eps$. They also give a new, simpler algorithm based on connections to uniform stability that yields the same bound as ours in the single-query setting. 
\else
In \cite{dwork2018privacy}, Dwork and Feldman consider the same problem, but focus on the sample complexity of a single query. In the agnostic PAC setting, their accuracy guarantee does not suffer from the constant blow-up in the approximation error as in the results of \cite{bassily2018model}. However, their sample complexity upper bound scales as $\approx \vc(\cH)/\eps\,\alpha^3$, which is sub-optimal.
Assuming $\epsilon=\Theta(1),$ our bound in the single-query setting (i.e., $m=1$) is essentially optimal as it nearly matches the standard non-private sample complexity in the agnostic PAC model. In an independent work \cite{Dagan-Feldman2019}, Dagan and Feldman further study the connections between uniform stability and differential privacy in the context of PAC learning, and give a new algorithm that yields a sample complexity bound of $\approx \vc(\cH)/\alpha^2 + \vc(\cH)^2/\eps\,\alpha$ in the single-query setting. Their bound exhibits the optimal dependence on $\epsilon$, and when $\vc(\cH)\leq \eps/\alpha,$ it is tighter than our bound by a factor of $1/\eps$. They also give a new, simpler algorithm based on connections to uniform stability that yields the same bound as ours in the single-query setting. 

\fi

Prior to the work of \cite{bassily2018model, dwork2018privacy}, there have been several works that considered similar problem settings, \ifthenelse{\alt=1}{\cite[e.g.][]{hamm2016learning, papernot2017semi, papernot2018scalable}}{e.g., \cite{hamm2016learning, papernot2017semi, papernot2018scalable}}. The last two references gave different algorithms and offered extensive empirical evaluation, however, they do not provide any formal accuracy guarantees.

\section{Preliminaries}\label{sec:prelim}
\paragraph{Notation:~} For classification tasks we denote the space of feature vectors by $\cX$, the set of labels by $\cY$, and the data universe by $U = \cX \times \cY$. 
A function $h : \cX \rightarrow \cY$ is called a hypothesis and it labels data points in the feature space $\cX$ by either $0$ or $1$ i.e. $\cY = \{0,1\}$. A set of hypotheses $\cH \subseteq \{0,1\}^\cX$ is called a hypothesis class. The VC dimension of $\cH$ is denoted by $\vc(\cH)$. 
We use $\cD$ to denote a distribution defined over $U = \cX\times \cY$, and $\cD_\cX$ to denote the marginal distribution over $\cX$. A sample dataset of $n$ i.i.d. draws from $\cD$ is denoted by $S = \{(x_1,y_1),\cdots,(x_n,y_n)\}$, where $x_i \in \cX$ and $y_i \in  \cY$. 

\paragraph{Expected error: }The expected error of a hypothesis $h : \cX \rightarrow \cY$ with respect to a
distribution $\cD$ over $U$ is defined by $\err(h;\cD) \triangleq \ex{(x,y)\sim \cD}{\ind(h(x) \neq y)}$. The excess expected error is defined as $\err(h;\cD) - \min\limits_{h\in \cH}\err(h; \cD) $.

\paragraph{Empirical error: }The empirical error of a hypothesis $h : \cX \rightarrow \cY$ with respect to a labeled set $S$ is denoted by $\widehat{\err}(h;S) ~\triangleq \frac{1}{n}\sum_{i=1}^{n} \ind(h(x_i) \neq y_i)$.

\noindent The problem of minimizing the empirical error on a dataset is known as Empirical Risk Minimization (ERM). We use $h_S^\erm$ to denote the hypothesis that minimizes the empirical error with respect to a dataset $S$, $h_S^\erm \triangleq \argmin\limits_{h \in \cH} \widehat{\err}(h;S)$.

\paragraph{Expected disagreement: }The expected disagreement between a pair of hypotheses $h_1$ and $h_2$ with respect to a distribution $\cD_\cX$ is defined as $\dis(h_1,h_2;\cD_\cX)~\triangleq\ex{x \sim \cD_\cX}{\ind(h_1(x))\neq h_2(x))}.$

\paragraph{Empirical disagreement: }The empirical disagreement between a pair of hypotheses $h_1$ and $h_2$ w.r.t. an unlabeled dataset \mbox{$S_u =\{x_1,\ldots,x_n\}$} is defined as \mbox{$\widehat{\dis}(h_1,h_2;S_u) \triangleq \frac{1}{n}\sum\limits_{i=1}^{n}\ind(h_1(x_i))\neq h_2(x_i))$}.

\paragraph{Realizable setting: }In the realizable setting of the PAC model, there exists a $h^* \in \cH$ such that $\err(h^*;\cD) = 0$ i.e., the true labeling function is assumed to be in $\cH$. In this case, the distribution $\cD$ can be described by $\cD_\cX$ and the hypothesis $h^* \in \cH$. Such a distribution $\cD$ is called \textit{realizable} by $\cH$. Hence, for realizable distributions, the expected error of a hypothesis $h$ will be denoted as $\err(h;(\cD_\cX,h^*)) \triangleq \ex{x\sim \cD_\cX}{\ind(h(x) \neq h^*(x))}$. 

\begin{defn}[Differential Privacy \ifthenelse{\alt=1}{\citep{DKMMN06,DMNS06}}{\cite{DKMMN06,DMNS06}}]
Let $\epsilon,\delta >0$. A (randomized) algorithm $M:U^n \rightarrow \cR$ is $(\eps,\delta)$-differentially private if for all pairs of datasets $S,S'\in U^n$ that differs in exactly one data point, and every measurable $\cO \subseteq \cR$, with probability at least $1-\delta$ over the coin flips of $M$, we have: $$\Pr \left(M(S) \in \cO \right) \leq e^\eps \cdot \Pr \left(M(S') \in \cO \right) +\delta.$$
\end{defn} 

We study private classification algorithms that take as input a private labeled dataset $S \sim \cD^n$, and a sequence of classification queries $\cQ = (\tx_1,\ldots,\tx_m) \sim \cD^{m}_{\cX}$, defined by $m$ unlabeled feature-vectors from $\cX$, (where $m$ is an input parameter), and output a corresponding sequence of predictions, i.e., labels, $(\hy_1,\ldots,\hy_m)$.
Here, we assume that the classification queries come one at a time and the algorithm is required to generate a label for the current query before seeing and responding to the next query. The goal is: i) after answering $m$ queries the algorithm should satisfy $(\epsilon,\delta)$-differential privacy, and ii) the labels generated should be $(\alpha,\beta)$-accurate with respect to a hypothesis class $\cH$: a notion of accuracy which we formally define shortly. We give a generic description of the above classification paradigm in Algorithm~\ref{Alg:seqClass} below (denoted as $\cA_\prvcl$).

 	\begin{algorithm}
 		\caption{$\cA_\prvcl$: Private Classification-Query Release Algorithm}
 		\begin{algorithmic}[1]
 			\REQUIRE Private dataset: $S \in (\cX \times \cY)^n$, upper bound on the number of queries: $m$, online sequence of classification queries: $\cQ=(\tx_1, \tx_2, \ldots ,\tx_m )$, ~hypothesis class: $\cH$, ~privacy parameters $\epsilon,\delta >0$, ~accuracy: $\alpha$, and failure probability: $\beta$
 			\FOR{$j=1, \ldots, m$}
	        	\STATE $\hy_{j}\leftarrow \mathsf{PrivLabel}(S,\cH,\left\{(\tx_i,\hy_i)\right\}_{i=1}^{j-1}, \tx_j)$ 
	        	\COMMENT{Generic procedure that, given $S,\cH$, the history $\left\{(\tx_i,\hy_i)\right\}_{i=1}^{j-1}$, and the current query $\tx_j$, generates a label $\hy_j$}\label{step:privclass}
	        	\STATE Output $\hy_j$
		    \ENDFOR
	
 		\end{algorithmic}
 		\label{Alg:seqClass}
 	\end{algorithm}
The algorithm $\cA_\prvcl$ invokes a procedure $\mathsf{PrivLabel}$, which is a generic classification procedure that given the input private training set $S$, the knowledge of hypothesis class $\cH$, and the previous queries and outputs, it generates a label for an input query (feature-vector) $\tx \in \cX$.  
\begin{defn}[\boldmath$(\epsilon,\delta,\alpha,\beta,n,m)$-Private Classification-Query Release Algorithm]\label{def:PSC}
Let $\cH$ be a hypothesis class $\cH \subseteq \{0,1\}^\cX$. Let $\epsilon,\delta,\alpha,\beta \in(0,1)$.  A randomized algorithm $\cA$ (whose generic format is described in Algorithm~\ref{Alg:seqClass}) is said to be an $(\epsilon,\delta,\alpha,\beta,n,m)$-$\pcqr$ (private classification-query release) algorithm for $\cH$, if the following conditions hold:
\begin{enumerate}
    \item For any sequence $\cQ \in \cX^m$, $\cA$ is $(\epsilon,\delta)$-differentially private with respect to its input dataset.
    \item For every distribution $\cD$ over $\cX \times \cY$, given a dataset $S \sim \cD^n$ and a sequence $V \triangleq ((\tx_1,\ty_1),\ldots,\allowbreak (\tx_m,\ty_m)) \sim \cD^m$ (where $\tx_i$'s are the queried feature-vectors in $\cQ$ and $\ty_i$'s are their true hidden labels), $\cA$ is $(\alpha,\beta)$-accurate with respect to $\cH$, where our notion of $(\alpha,\beta)$-accuracy is defined as follows: With probability at least $1-\beta$ over the choice of $S,~ V$, and the internal randomness in $\mathsf{PrivLabel}$ (Step~\ref{step:privclass} in Algorithm~\ref{Alg:seqClass}), we have
    $$  \frac{1}{m}\sum_{j=1}^{m} {\ind(\hy_j \neq \ty_j)} \leq \alpha + \gamma, $$
    where $\gamma \triangleq \min\limits_{h\in \cH}\err(h; \cD).$
\end{enumerate}
In the realizable setting, we have an analogous definition where $\gamma=0$. In this case, we say that the algorithm is a $\pcqr$ algorithm for $\cH$ in the realizable setting.
\end{defn}


\ifthenelse{\alt=1}{\subsection{Previous work on private classification-query release \citep{bassily2018model}}}{\subsection{Previous work on private classification-query release \cite{bassily2018model}}}\label{sec:Apriv}

\ifthenelse{\alt=1}{\cite{bassily2018model} give}{In \cite{bassily2018model}, they give} a construction for a $\pcqr$ algorithm (referred to as $\cA_\samp$), which combines the sub-sample-aggregate framework \ifthenelse{\alt=1}{\citep{NRS07,ST13}}{\cite{NRS07,ST13}} with the sparse-vector technique \ifthenelse{\alt=1}{\citep{DR14}}{\cite{DR14}}. \ifthenelse{\alt=1}{They}{Bassily et al. \cite{bassily2018model}} provide formal privacy and accuracy guarantees with sample complexity bounds for $\cA_\samp$ in both the realizable and agnostic settings of the PAC model. 
As in the sparse-vector technique, one important input parameter to $\cA_\samp$ is cut-off paramter $T$, which gives bound on the number of the so-called ``unstable queries'' that $\cA_{\samp}$ can answer before the privacy budget is consumed. We formally describe $\cA_\samp$ and the notion of ``unstable queries'' in Appendix~\ref{apndx:subsamp} for completeness.
Here, we restate the results \ifthenelse{\alt=1}{by}{of} \cite{bassily2018model} for the realizable and agnostic settings. 


\begin{lem}[Realizable Setting: follows from Theorems~3.2 \& 3.4, \ifthenelse{\alt=1}{\citep{bassily2018model}}{\cite{bassily2018model}}]\label{lem:privPAC}
	Let $\epsilon,\delta >0$ and, $\alpha, \beta \in (0,1).$ Let $\cH$ be a hypothesis class with $\vc(\cH)=d$. Suppose that $\cB$ in $\cA_{\samp}$ is a PAC learner for $\cH$. Let $\cD$ be any distribution over $U$ that is realizable by $\cH$. There is a setting for the cut-off parameter $T=\max\left(1,~\tilde{O}\left(m\,\alpha \right)\right)$ such that $\cA_\samp$ is an $(\epsilon,\delta,\alpha,\beta,n,m)$-$\pcqr$ algorithm for $\cH$ in the realizable setting where the private sample size is $n = \tilde{O}\left(\frac{d}{\epsilon\,\alpha}\cdot\max\left(1,\sqrt{m\,\alpha}\right)\right)$.
\end{lem}

In the agnostic setting, the accuracy guarantee of \ifthenelse{\alt=1}{\citep{bassily2018model}}{\cite{bassily2018model}} is not compatible with Definition~\ref{def:PSC}; the accuracy guarantee therein has a sub-optimal dependency on the approximation error, $\gamma$ (where $\gamma \triangleq \min\limits_{h\in \cH}\err(h; \cD))$. In particular, their result entails a blow-up in $\gamma$ by a constant factor ($\approx 3$). This significantly limit the applicability of this result in scenarios where $\gamma \gg \alpha$. 



\begin{lem}[Agnostic Setting: follows from Theorems~3.2 \& 3.5, \ifthenelse{\alt=1}{\citep{bassily2018model}}{\cite{bassily2018model}}]\label{lem:AgPrivPAC}
	Let $\epsilon, \delta,\alpha, \beta \in (0,1)$. Let $\cH$ be a hypothesis class with $\vc(\cH)=d$. Suppose $\cB$ in $\cA_{\samp}$ is an agnostic PAC learner for $\cH$. Let $\cD$ be any distribution over $U$, and let $\gamma \triangleq \min\limits_{h\in \cH}\err(h; \cD)$. Let $S \sim \cD^n$ denote the input private sample to $\cA_\samp$. Let $V \triangleq ((\tx_1,\ty_1),\ldots,(\tx_m,\ty_m)) \sim \cD^m,$ where $\tx_i$'s are the queried feature-vectors in $\cQ$ and $\ty_i$'s are their true (hidden) labels. Let $(y_1^\prv, \ldots, y_m^\prv)$ denote the output labels of $\cA_\samp$. There is a setting for the cut-off parameter $T=\max\left(1,~ \tilde{O}\left(m\left(\alpha+\gamma\right)\right)\right)$ such that:
	1) $\cA_\samp$ is $(\epsilon,\delta)$-differentially private with respect to the input training set; 
    2) when the private sample is of size $n = \tilde{O}\left(\frac{d}{\epsilon\,\alpha^2}\cdot\max\left(1,\sqrt{m\,\alpha}\right)\right)$, then with probability at least $1-\beta$ over $S,~V$ and the randomness in $\cA_\samp$, we have:$$\frac{1}{m}\sum_{j=1}^{m} {\ind(y_j^\prv \neq \ty_j)} \leq \alpha + 3\gamma.$$
\end{lem}

\section{Private Release of Classification Queries in the Agnostic PAC Setting}

In this section we give an improved construction for the private classification-query release algorithm \ifthenelse{\alt=1}{of}{in} \cite{bassily2018model} in the agnostic setting. 
Our construction can privately answer up to $m$ queries with excess classification error $\alpha$, and input sample size $\tilde{O}\left(\frac{\vc(\cH)}{\epsilon~\alpha^2}\cdot\max\left(1,\sqrt{m}~\alpha^{3/2}\right)\right)$, (where $\tilde{O}$ hides $\log$ factors of $m,\frac{1}{\alpha},\frac{1}{\delta},\frac{1}{\beta}$). 
Comparing to the result \ifthenelse{\alt=1}{by}{of} \cite{bassily2018model} for the agnostic setting, where the private sample size is $\approx \frac{\vc(\cH)}{\eps~\alpha^2}\cdot\max(1,\sqrt{m\alpha})$ (Lemma~\ref{lem:AgPrivPAC}), our sample complexity bound is tighter by a factor of $\approx \sqrt{m\alpha}$ when $\frac{1}{\alpha} \leq m < \frac{1}{\alpha^3} $, and it is tighter by a factor of $\approx \frac{1}{\alpha}$ when $m \geq \frac{1}{\alpha^3}$.  
\subsection*{Overview}
Our construction is made up of two phases. The first phase is a pre-processing phase in which a subset $S',$ of the input private sample $S,$  is \emph{relabeled} using a ``good'' hypothesis $\hh\in\cH$ to obtain a new sample $S''$. This phase is a reenactment of the elegant technique due to \ifthenelse{\alt=1}{\cite{BeimelNS14}}{Beimel et al. \cite{BeimelNS14}}, which was called \emph{LabelBoost Procedure} therein. By construction $\hh$ is chosen such that its empirical error is close to that of the ERM hypothesis. Hence, we can formally show that when input sample size is sufficiently large, $\hh$ attains low excess error. Moving forward, one may view $\hh$ as if it is the true labeling hypothesis, and hence the agnostic setting can be reduced to the realizable setting. In Section~\ref{sec:Relabel}, we  describe this pre-processing phase and state its guarantees.

Now as we reduced the problem to the realizable setting, in the next phase we invoke the techniques of \ifthenelse{\alt=1}{\citep{bassily2018model}}{\cite{bassily2018model}}. In the second phase, the relabeled training set $S''$ is used to provide input training examples to $\cA_{\samp}$ (described in Section~\ref{sec:Apriv}). Note that $S''$ is no longer i.i.d. from the original distribution. We form a new dataset $\hS$ by sampling data points uniformly with replacement from $S''$ and then feed $\hS$ to $\cA_{\samp}$ as input. This new training set $\hS$ is now i.i.d. from the empirical distribution of $S''$. Via a uniform-convergence argument (see Claim~\ref{clm:disRate}), we can show that that this re-sampling step does not impact our desired accuracy guarantees. We also need to carefully calibrate the privacy parameters of $\cA_\samp$ to take into account the fact that $\hS$ may contain repetitions of the elements in $S''$. Algorithm $\cA_{\samp}$ uses $\hS$ to privately generate labels for an online sequence of classification queries. We formally show that for any setting of the target parameters (accuracy, privacy, and total number of queries), there is a sufficient size for the original input sample $S$ such that our construction attains the desired accuracy and privacy guarantees w.r.t. the entire sequence of queries. We formally describe our construction and provide formal analysis for its privacy and accuracy guarantees in Section~\ref{sec:AgSecCl}.

\subsection{From the agnostic to the realizable setting: A generic reduction}\label{sec:Relabel}

In this section, we describe the pre-processing procedure, denoted as $\cA_{\relb}$ (given by Algorithm~\ref{Alg:relabel} below), which follows from the relabeling technique devised by \ifthenelse{\alt=1}{\cite{BeimelNS14}}{Beimel et al. in \cite{BeimelNS14}}. 

{The algorithm $\cA_\relb$ operates on a private labeled dataset $S' \sim \cD^{n'}$ and on a hypothesis class $\cH$.  
Let $S'_u$ denote the unlabeled version of $S',$ i.e., $S'_u = \{x_1,\ldots,x_{n'}\}$, and $\prod_\cH(S'_u)$ denote the set of all possible dichotomies that can be generated by $\cH$ on the set $S'_u$. First the algorithm chooses a finite subset $\widetilde{H}$ of $\cH$ such that each dichotomy in $\prod_\cH(S'_u)$ is represented by one of the hypotheses in $\widetilde{H}$. 
Note that by Sauer’s lemma \ifthenelse{\alt=1}{\citep[see][]{sauer1972density}}{(see \cite{sauer1972density})}, the size of $\widetilde{H}$ is $O\left((n'/d)^d\right)$, where $d=\vc(\cH)$. Next, $\cA_{\relb}$ chooses a hypothesis $\whh$ using the exponential mechanism with privacy parameter $\tilde{\epsilon}=1$ and a score function $q(S', h) = - \herr(h; S')$. Finally, $\cA_\relb$ uses $\whh$ to rebalel $S'_u$, and outputs this labeled set $S''$.}

	
	\begin{algorithm}[H]
		\caption{$\A_{\relb}$: Relabel Procedure}
		\begin{algorithmic}[1]
			\REQUIRE Private dataset: $S'\in \left(\cX\times \cY\right)^{n'}$, ~a hypothesis class: $\cH$
			
			\STATE $\widetilde{H} \leftarrow \emptyset$
			
			\STATE Let $S'_u = \{x_1 , \ldots, x_{n'} \}$ be the unlabeled version of $S'$.
			
			\STATE For every $\left(y_1,\ldots,y_{n'}\right) \in \prod_\cH(S'_u) = \{\left(h(x_1),\ldots,h(x_{n'})\right): h \in \cH\}$, ~add to $\widetilde{H}$ any arbitrary hypothesis $h \in \cH$ s.t. ~$h(x_i)=y_i, \forall i \in [n']$.
			\STATE Use the exponential mechanism with inputs $S',~\widetilde{H}$, privacy parameter $\tilde{\eps}=1$, and a score function $q(S', h)\triangleq -\herr(h; S')$ to select $\whh$ from $\widetilde{H}$. \label{Stp:expMech}
			
			\STATE Relabel $S'_u$ using $\whh$, and denote this relabeled dataset as $S''$.
			
			\STATE Output $S''$.
			
		\end{algorithmic}
		\label{Alg:relabel}
	\end{algorithm}
The following lemmas give the privacy and accuracy guarantees of $\cA_\relb$. Lemma~\ref{lem:RelabelDP} follows directly from \ifthenelse{\alt=1}{\citep{BeimelNS14}}{\cite{BeimelNS14}}. We prove Lemma~\ref{lem:RelabelAcc} in Appendix~\ref{apndx:prfRelb}.
	\begin{lem}[Lemma 4.1 in \ifthenelse{\alt=1}{\citep{BeimelNS14}}{\cite{BeimelNS14}} restated]\label{lem:RelabelDP}
		Let $\cA$ be an $\left(1,\delta\right)$- differentially private algorithm. Let $\cB$ be an algorithm that on input dataset $S'$ invokes $\cA$ on the output of $\cA_\relb(S',\cH)$. Then, $\cB$ is $\left(4,4e\delta\right)$- differentially private. 
	\end{lem}
	\begin{lem}\label{lem:RelabelAcc}
		Let $\cH$ be a hypothesis class with $\vc(\cH) = d$. Let $\alpha,\beta \in (0,1)$. Let $\cD$ be an arbitrary distribution over $\cX\times \cY$, and $S' \sim \cD^{n'}$ be an input dataset to $\cA_\relb$, where $n' \geq 256\frac{(d + \log(3/\beta))}{~\alpha^2}$. With probability at least $1-\beta$, hypothesis $\widehat{h}$ (generated in Step~\ref{Stp:expMech} of $\A_\relb$) satisfies the following: 
		$$\err\left(\whh; \cD\right)- \err\left(h_{S'}^\erm; \cD\right)\leq \alpha,$$
		where $h_{S'}^\erm$ is the ERM hypothesis w.r.t. the input sample $S'$.
	\end{lem}

\subsection{A Private Classification-Query Release Algorithm}\label{sec:AgSecCl}

In this section, we describe our $\pcqr$ algorithm $\cA_\Agprv$ (Algorithm~\ref{Alg:AgnSeqClas} below) that combines the two techniques given by $\cA_\relb$, and $\cA_\samp$. As a $\pcqr$ algorithm, $\cA_\Agprv$ takes as input: a private dataset $S \sim \cD^n$, the number of queries $m$, an online sequence of classification queries $\cQ=(\tx_1,\ldots,\tx_m) \sim \cD_{\cX}^m$, \mbox{a hypothesis} class $\cH$, as well as the desired privacy and accuracy parameters. Together with these, $\cA_\Agprv$ also has oracle access to a PAC learner $\cB_\pac$ for $\cH$. 
First, we randomly sample a subset $S'$ of $S$ of size $n'$, where $n' \approx \epsilon n$, and invoke $\cA_\relb$ on $S'$ and $\cH$. This sampling step is used to boost the privacy guarantee of $\cA_\Agprv$. Note that, dataset $S''$ (output by $\cA_\relb$) is relabeled using hypothesis $\whh \in \cH$. In order to ensure that our input to the next stage is i.i.d., we sample $n'$ points uniformly with replacement from $S''$ to form a new dataset $\hS$ (i.e., $\hS$ is made up of $n'$ i.i.d. draws from the empirical distribution of $S''$). Next, we invoke $\cA_\samp$ in the realizable setting on the dataset $\hS$, $m$, $\cQ$, and $\cB_\pac$ as inputs. We set the cut-off parameter of $\cA_\samp$ as $T = \max\left(1,~\tilde{O}(m\alpha)\right)$, where $\alpha$ is the accuracy parameter of $\cB_\pac$. The privacy parameters to $\cA_\samp$ are set to $(\heps,\hdel)$ defined in Step~\ref{Stp:assign} of $\cA_\Agprv$. This is needed to ensure $(\epsilon,\delta)$-differential privacy for the entire construction. Finally, we output the sequence of private labels $\{y_1^\prv,\dots,y_m^\prv\}$ generated by $\cA_\samp$ for the input sequence of queries.


\begin{algorithm}[H]
 		\caption{$\A_\Agprv$: Private Agnostic-PAC Classification-Query Release Algorithm}
 		\begin{algorithmic}[1]
 			\REQUIRE Private dataset: $S \in (\cX \times \cY)^n$, upper bound on the number of queries: $m$, online sequence of classification queries: $\cQ=(\tx_1, \tx_2, \ldots, \tx_m )$, ~a hypothesis class: $\cH$, ~oracle access to non-private learner: $\cB_\pac$ for $\cH$, ~privacy parameters: $\epsilon,\delta >0$, ~accuracy parameter: $\alpha$, and,~ failure probability: $\beta$
 			\STATE $~ n' \leftarrow \frac{\epsilon}{56} n, \quad ~T \leftarrow \max\left(1,~ \frac{1}{8}\,m\alpha + \frac{1}{4}\sqrt{3\,m \alpha \log\left(\frac{m}{\beta}\right)}\right),$\\
 			$~\epsilon' \leftarrow \alpha\max\left(1,\sqrt{m\alpha}\right), \quad ~\heps \leftarrow \frac{1}{\log(2/\delta)}\min \left(1,~\epsilon'\right), \quad ~ \hdel \leftarrow \frac{\delta}{2~e^{\min (1,~\epsilon')}~\log(2/\delta)}$ \label{Stp:assign}
 			
 			
 			\STATE Uniformly sample without replacement a subset $S'$ of $n'$ data points from $S$  \label{Stp:smplS}
			
 			\STATE $S'' \leftarrow \cA_\relb(S',\cH).$ \label{Stp:relb}
 			\STATE $\hS \leftarrow$ Uniformly sample $n'$ points from $S''$ with replacement. \label{Stp:unifSamp}
  			\STATE  Output $(y_1^\prv, \ldots, y_m^\prv) \leftarrow \cA_\samp(\hS,m,\cQ,\cB_\pac,T,\heps,\hdel,\beta)$ \label{Stp:AsubSamp}.
 		\end{algorithmic}
 		\label{Alg:AgnSeqClas}
\end{algorithm}
\noindent We formally state the main result of this section in the following theorem.
\begin{thm}\label{thm:AgPriv}
	
		Let $\cH$ be a hypothesis class with $\vc(\cH)=d$. For any $\epsilon,\delta,\alpha,\beta \in (0,1)$, $\cA_\Agprv$ (Algorithm~\ref{Alg:AgnSeqClas}) is an $(\epsilon,\delta,\alpha,\beta,n,m)$-$\pcqr$ algorithm for $\cH$, where private sample size $$n = O\left(\frac{\left(d \,\log\left(\frac{1}{\alpha}\right) + \log\left(\frac{m}{\beta}\right)\right)\, \log^{3/2}\left(\frac{2}{\delta}\right)\, \log\left(\frac{m\alpha}{\min(\delta,\beta/2)}\right)}{\epsilon ~\alpha^2}\cdot\max \left(1,\sqrt{m}\, \alpha^{3/2}\right)\right),$$
and number of queries $m=\Omega\left(\frac{\log(1/\alpha\beta)}{\alpha}\right)$.
\end{thm}
We will prove the theorem via the following lemmas that establish the privacy and accuracy guarantees of $\cA_\Agprv$.
\begin{lem}[Privacy Guarantee of $\cA_\Agprv$]\label{lem:AgSC-DP}
$\cA_\Agprv$ is $\left(\epsilon,\delta\right)$-differentially private (with respect to its input dataset). 
\end{lem}
\begin{proof}
Fix the randomness in dataset $S'$ due to sampling in Step~\ref{Stp:smplS} of $\A_\Agprv$. Let $\cR(\cdot)$ denote the uniform sampling procedure in Step~\ref{Stp:unifSamp} in $\cA_\Agprv$; that is, Step~\ref{Stp:unifSamp} can be written as $\hS \leftarrow \cR(S'')$. Note that Steps~\ref{Stp:unifSamp}-\ref{Stp:AsubSamp} in $\cA_\Agprv$ can now be expressed as a composition $\cR \circ \cA_\samp$, where $\cR \circ \cA_\samp(\cdot) \triangleq \cA_\samp(\cR(\cdot))$.

Let $\epsilon^*=\min \left(1,~\epsilon'\right)$, where $\epsilon'=\alpha\max\left(1,\sqrt{m\alpha}\right)$ (as defined in Step~\ref{Stp:assign} of $\A_\Agprv$). 
Note that the input to $\cR \circ \cA_\samp$ is the dataset $S''$, which is the output of $\cA_\relb$. Note that if we can show that $\cR \circ \cA_\samp$ is $(\epsilon^*,\delta)$-differentially private, then, from Lemma~\ref{lem:RelabelDP}, it follows that $\cA_\relb \circ \cR \circ \cA_\samp$ is $(\epsilon^*+3,4e\delta)$-differentially private. Next, by taking into account account the randomness due to sampling in Step~\ref{Stp:smplS}, then by privacy amplification via sampling \ifthenelse{\alt=1}{\citep{KLNRS08, li2012sampling}}{\cite{KLNRS08,li2012sampling}}, it follows that $\cA_\Agprv$ is $\left(\epsilon,\delta\right)$-differentially private. Hence, it remains to show that $\cR \circ \cA_\samp$ is $(\epsilon^*,\delta)$-differentially private with respect to $S''$.

Let $S''_1$ and $S''_2$ be neighboring datasets. W.l.o.g., assume that $S''_1$ and $S''_2$  differ in index $j \in [n']$. Let $r$ be the number of times the $j$-th index is sampled by $\cR$. By the definition of $\cR$, and Chernoff bound, w.p. $\geq 1-\delta/2$, we have $r \leq \log(2/\delta)$.

\ifthenelse{\alt=1}{Using the result in \cite[Theorem 3.1][]{bassily2018model}}{Using the result in \cite[Theorem 3.1]{bassily2018model}}, $\cA_\samp$ is $(\heps,\hdel)$-differentially private with respect to $\hS$. 
Conditioned on $r\leq\log(\frac{2}{\delta})$ and by the notion of group privacy we have, $\cR \circ \cA_\samp$ is $(r\heps~,~re^{r\heps}~\hdel)$-differentially private. 
Hence, by the above high probability bound on the event $r\leq\log(\frac{2}{\delta}),$ we conclude that $\cR \circ \cA_\samp$ is $(\min \left(1,~\epsilon'\right),\delta)$-differentially private.

\end{proof}

\begin{lem}[Accuracy Guarantee of $\cA_\Agprv$]\label{lem:AgSC-Acc}
    Let $\cH$ be a hypothesis class with $\vc(\cH)=d$. Let $\cB_\pac$ (invoked by $\cA_\samp$) be a PAC learner for $\cH$ (in the realizable setting). Let $\cD$ be any distribution over $\cX \times \cY$, and let $\gamma \triangleq \min\limits_{h\in \cH}\err(h; \cD)$.  ~Let $S \sim \cD^n$ denote the input private sample to $\cA_\Agprv$, where 
    $$n = O\left(\frac{\left(d \,\log\left(\frac{1}{\alpha}\right) + \log\left(\frac{m}{\beta}\right)\right)\, \log^{3/2}\left(\frac{2}{\delta}\right)\, \log\left(\frac{m\alpha}{\min(\delta,\beta/2)}\right)\max\left(1,\sqrt{m}\,\alpha^{3/2}\right)}{\epsilon ~\alpha^2}\right),$$ and $m\geq 8\,\frac{\log(1/\alpha\beta)}{\alpha}$. Let $(\ty_1,\ldots,\ty_m)$ denote the corresponding true (hidden) labels for $\cQ$. Then, w.p. at least $1-\beta$ (over the choice of $S,~\cQ,$ and the randomness in $\cA_\Agprv$), we have: $$\frac{1}{m}\sum_{j=1}^{m} {\ind(y_j^\prv \neq \ty_j)} \leq \alpha + \gamma.$$ 
\end{lem}
In the proof of Lemma~\ref{lem:AgSC-Acc} we will use the following claim. We defer its proof after the proof of the lemma.
\begin{claim}\label{clm:disRate}
	Let $\cH$ be a hypothesis class with $\vc(\cH)=d$. Let $S_u$ be an an unlabeled training set of size $n_o$, where $n_o \geq 50\frac{d\,\log(1/\alpha) + \log (1/\beta')}{\alpha^2}$. Then, with probability at least $1-\beta'$ for any $h_1,h_2 \in \cH$, we have $\left|{\dis}(h_1,h_2;\cD_\cX) - \widehat{\dis}(h_1,h_2;S_u) \right| \leq \alpha$. (Recall that ${\dis}(h_1,h_2;\cD_\cX)$ and $\widehat{\dis}(h_1,h_2;S_u)$ are the expected and empirical disagreement rates, respectively, as defined in Section~\ref{sec:prelim}.)
	
\end{claim}


\ifthenelse{\alt=0}{
\begin{proof}[Proof of Lemma~\ref{lem:AgSC-Acc}]~
Consider the description of $\cA_\relb$ in Algorithm~\ref{Alg:relabel}. Note that, hypothesis $\whh \in \cH$ selected in Step~\ref{Stp:expMech} of $\A_\relb$ is used to generate labels of $S''$ (output dataset of $\cA_\relb$). Note that by choosing $n$ to be sufficiently large, we can ensure that the size of $S''$ is given by  
$$n' = 8000\,\frac{\left(d\,\log\left(\frac{1}{\alpha}\right) + \log\left(\frac{m}{\beta}\right)\right)\polylog\left(m,\frac{1}{\delta},\frac{1}{\beta}\right)}{\alpha^2}\cdot\max \left(1,\sqrt{m}\,\alpha^{3/2}\right).$$
Let $\cD_{S''}$ denote the empirical distribution induced by $S''$. Note that $\err(\whh;\cD_{S''}) = 0$. In $\cA_\Agprv$, dataset $\hS$ (input to $\cA_\samp$) is created by $n'$ i.i.d. draws from $\cD_{S''}$.

 From the description of $\cA_\samp$ (Algorithm~\ref{Alg:binClas}), $\cA_\samp$ splits $\hS$ into $k$ equal-sized sub-samples, where $k = \tilde{O}\left(\frac{\sqrt{T}}{\heps}\right)$. Here $T$ is the input cut-off parameter of $\cA_\samp$ whose setting is given in Step~\ref{Stp:assign} of $\cA_\Agprv$. Note that since $m=\Omega\left(\frac{\log(1/\alpha\beta)}{\alpha}\right),$ we have $T=O(m\alpha)$. Each sub-sample is then fed separately as an input to $\cB_\pac$. For each input sub-sample, $\cB_\pac$ outputs a classifier $h_j, j \in [k]$. Hence we end up with an ensemble of classifiers $h_{1},\cdots,h_{k}$. Note that the size of the input sub-sample to $\cB_\pac$ is $\frac{n'}{k}$. Observe that
\begin{align*}
    n' = 8000\,\frac{\left(d\,\log\left(\frac{1}{\alpha}\right) + \log\left(\frac{m}{\beta}\right)\right)\,\log^{3/2}\left(\frac{2}{\delta}\right)\, \log\left(\frac{m\alpha}{\min(\delta,\beta/2)}\right)}{\alpha^2}\cdot\max \left(1,\sqrt{m} \,\alpha^{3/2}\right),
\end{align*}
and the number of sub-samples $k$ is set in Step~\ref{Stp:kAssn} of $\cA_\samp$ as follows
\begin{align*}
    k = O\left(\frac{\sqrt{m\alpha\log(\frac{2}{\delta})}\cdot\log\left(\frac{m\alpha}{\min(\delta, \beta/2)}\right)}{\heps}\right)
\end{align*}

\noindent Hence, using the setting of $\heps$ in Step~\ref{Stp:assign} of $\cA_\Agprv$, we have
\begin{align*}
    \frac{n'}{k} &= \Omega\left(\frac{\left(d\,\log\left(\frac{1}{\alpha}\right) + \log\left(\frac{m}{\beta}\right)\right)}{\sqrt{m}\,\alpha^{5/2}}\cdot\min\left(1,\sqrt{m}\, \alpha^{3/2}\right)\cdot\max \left(1,\sqrt{m}\, \alpha^{3/2}\right)\right)\\
    &= \Omega\left(\frac{\left(d \,\log\left(\frac{1}{\alpha}\right) + \log\left(\frac{m}{\beta}\right)\right)}{\alpha}\right) =\Omega\left(\frac{\left(d \,\log\left(\frac{1}{\alpha}\right) + \log\left(\frac{16k}{\beta}\right)\right)}{\alpha}\right).
\end{align*}
\noindent By standard results in learning theory, it is easy to see that the size of the input sub-sample to $\cB_\pac$ is sufficient for $\cB_\pac$ to PAC-learn $\cH$ with respect to $\cD_{S''}$ with accuracy $\frac{\alpha}{24}$ and confidence $\frac{\beta}{16k}$.

Fix any $j \in [k]$. Using the above fact about $\cB_\pac$, ~w.p. at least $1-\frac{\beta}{16k}$, ~~$\err(h_j;\cD_{S''}) \leq \frac{\alpha}{24}$.
Since $\cD_{S''}$ is the empirical distribution of $S''$, equivalently, we have $\widehat{\dis}(h_j,\whh; S''_u) \leq \frac{\alpha}{24} $, where $S''_u$ is the unlabeled version of $S''$.  
Note that the size of $S''$ is $n' \geq  7200\frac{\left(d\,\log\left(\frac{1}{\alpha}\right) + \log\left(\frac{8k}{\beta}\right)\right)}{\alpha^2}$. Hence, by Claim~\ref{clm:disRate}, it follows that w.p. $\geq 1-\frac{\beta}{8k}$,  ~$\dis(h_j,\whh;\cD_\cX) \leq \frac{\alpha}{12}$. Equivalently, w.p. $\geq 1-\frac{\beta}{8k}$,~~$ \err\left(h_j;(\cD_\cX, \whh)\right) \leq \frac{\alpha}{12}.$

From the above and the fact that the queries in $\cQ$ are i.i.d. from $\cD_\cX$, we invoke the same counting argument in the proof of \ifthenelse{\alt=1}{\citep[Theorem 3.2][]{bassily2018model}}{\cite[Theorem 3.2]{bassily2018model}} to show that w.p. $\geq 1-\frac{\beta}{4}$, the output labels of $\cA_\samp$ satisfy:
\begin{align}\label{eq:MisClrate}
    \frac{1}{m} \sum_{i=1}^{m} \ind\left(y_i^\prv \neq \whh(\tx_i)\right) &\leq \frac{\alpha}{4}.
\end{align}
Let $h_{S'}^\erm$ denote the ERM hypothesis with respect to the dataset $S'$ constructed in Step~\ref{Stp:smplS} of $\cA_{\Agprv}$. Note that Lemma~\ref{lem:RelabelAcc} implies that w.p. $\geq 1-\beta/4$, ~$\err(\whh; \cD)-\err\left(h_{S'}^\erm; \cD\right)\leq \alpha/4.$

Since the queries and their true labels $\big((\tx_1, \ty_1), \ldots, (\tx_m, \ty_m)\big)$ are drawn i.i.d. from $\cD$, then by Chernoff's bound and the fact that $m\geq 8\frac{\log(1/\beta)}{\alpha}$, we get that w.p. $\geq 1-\frac{\beta}{2},$ 
\begin{align}
\frac{1}{m} \sum_{i=1}^{m}\ind\left(\whh(\tx_i)\neq\ty_i\right) - \frac{1}{m}\sum_{i=1}^{m}\ind\left(h_{S'}^\erm(\tx_i)\neq\ty_i\right) &\leq \frac{\alpha}{2}.\label{eq:misclass-whh-herm}
\end{align}
Moreover, from the bound on $n'$ and using a basic fact from learning theory, w.p. $\geq 1-\beta/8$, the ERM hypothesis $h_{S'}^\erm$ satisfies: ~$\err\left(h_{S'}^\erm; \cD\right)\leq \alpha/8 + \gamma,$ where $\gamma=\min\limits_{h\in\cH}\err(h; \cD)$.
Again, since $\big((\tx_1, \ty_1), \ldots, (\tx_m, \ty_m)\big)$ are i.i.d. from $\cD$, then by Chernoff's bound and the fact that $m\geq 8\frac{\log(1/\beta)}{\alpha}$, w.p. $\geq 1-\beta/4$, we have
\begin{align}
    \frac{1}{m} \sum_{i=1}^{m} \ind\left(h_{S'}^\erm(\tx_i) \neq \ty_i\right)&\leq \frac{\alpha}{4} +\gamma. \label{eq:misclass-herm}
\end{align}
Now, using (\ref{eq:MisClrate}), (\ref{eq:misclass-whh-herm}), and (\ref{eq:misclass-herm}) together with a simple application of the triangle inequality and the union bound, we conclude that w.p. $\geq 1-\beta$,  ~~~ $\frac{1}{m}\sum_{j=1}^{m} {\ind\left(y_j^\prv \neq \ty_j\right)} \leq \alpha + \gamma.$
\end{proof}}
{\begin{proof}{\bf of Lemma~\ref{lem:AgSC-Acc}}~
Consider the description of $\cA_\relb$ in Algorithm~\ref{Alg:relabel}. Note that, hypothesis $\whh \in \cH$ selected in Step~\ref{Stp:expMech} of $\A_\relb$ is used to generate labels of $S''$ (output dataset of $\cA_\relb$). Note that by choosing $n$ to be sufficiently large, we can ensure that the size of $S''$ is given by  
$$n' = 8000\,\frac{\left(d\,\log\left(\frac{1}{\alpha}\right) + \log\left(\frac{m}{\beta}\right)\right)\polylog\left(m,\frac{1}{\delta},\frac{1}{\beta}\right)}{\alpha^2}\cdot\max \left(1,\sqrt{m}\,\alpha^{3/2}\right).$$
Let $\cD_{S''}$ denote the empirical distribution induced by $S''$. Note that $\err(\whh;\cD_{S''}) = 0$. In $\cA_\Agprv$, dataset $\hS$ (input to $\cA_\samp$) is created by $n'$ i.i.d. draws from $\cD_{S''}$.

 From the description of $\cA_\samp$ (Algorithm~\ref{Alg:binClas}), $\cA_\samp$ splits $\hS$ into $k$ equal-sized sub-samples, where $k = \tilde{O}\left(\frac{\sqrt{T}}{\heps}\right)$. Here $T$ is the input cut-off parameter of $\cA_\samp$ whose setting is given in Step~\ref{Stp:assign} of $\cA_\Agprv$. Note that since $m=\Omega\left(\frac{\log(1/\alpha\beta)}{\alpha}\right),$ we have $T=O(m\alpha)$. Each sub-sample is then fed separately as an input to $\cB_\pac$. For each input sub-sample, $\cB_\pac$ outputs a classifier $h_j, j \in [k]$. Hence we end up with an ensemble of classifiers $h_{1},\cdots,h_{k}$. Note that the size of the input sub-sample to $\cB_\pac$ is $\frac{n'}{k}$. Observe that
\begin{align*}
    n' = 8000\,\frac{\left(d\,\log\left(\frac{1}{\alpha}\right) + \log\left(\frac{m}{\beta}\right)\right)\,\log^{3/2}\left(\frac{2}{\delta}\right)\, \log\left(\frac{m\alpha}{\min(\delta,\beta/2)}\right)}{\alpha^2}\cdot\max \left(1,\sqrt{m} \,\alpha^{3/2}\right),
\end{align*}
and the number of sub-samples $k$ is set in Step~\ref{Stp:kAssn} of $\cA_\samp$ as follows
\begin{align*}
    k = O\left(\frac{\sqrt{m\alpha\log(\frac{2}{\delta})}\cdot\log\left(\frac{m\alpha}{\min(\delta, \beta/2)}\right)}{\heps}\right)
\end{align*}

\noindent Hence, using the setting of $\heps$ in Step~\ref{Stp:assign} of $\cA_\Agprv$, we have
\begin{align*}
    \frac{n'}{k} &= \Omega\left(\frac{\left(d\,\log\left(\frac{1}{\alpha}\right) + \log\left(\frac{m}{\beta}\right)\right)}{\sqrt{m}\,\alpha^{5/2}}\cdot\min\left(1,\sqrt{m}\, \alpha^{3/2}\right)\cdot\max \left(1,\sqrt{m}\, \alpha^{3/2}\right)\right)\\
    &= \Omega\left(\frac{\left(d \,\log\left(\frac{1}{\alpha}\right) + \log\left(\frac{m}{\beta}\right)\right)}{\alpha}\right) =\Omega\left(\frac{\left(d \,\log\left(\frac{1}{\alpha}\right) + \log\left(\frac{16k}{\beta}\right)\right)}{\alpha}\right).
\end{align*}
\noindent By standard results in learning theory, it is easy to see that the size of the input sub-sample to $\cB_\pac$ is sufficient for $\cB_\pac$ to PAC-learn $\cH$ with respect to $\cD_{S''}$ with accuracy $\frac{\alpha}{24}$ and confidence $\frac{\beta}{16k}$.

Fix any $j \in [k]$. Using the above fact about $\cB_\pac$, ~w.p. at least $1-\frac{\beta}{16k}$, ~~$\err(h_j;\cD_{S''}) \leq \frac{\alpha}{24}$.
Since $\cD_{S''}$ is the empirical distribution of $S''$, equivalently, we have $\widehat{\dis}(h_j,\whh; S''_u) \leq \frac{\alpha}{24} $, where $S''_u$ is the unlabeled version of $S''$.  
Note that the size of $S''$ is $n' \geq  7200\frac{\left(d\,\log\left(\frac{1}{\alpha}\right) + \log\left(\frac{8k}{\beta}\right)\right)}{\alpha^2}$. Hence, by Claim~\ref{clm:disRate}, it follows that w.p. $\geq 1-\frac{\beta}{8k}$,  ~$\dis(h_j,\whh;\cD_\cX) \leq \frac{\alpha}{12}$. Equivalently, w.p. $\geq 1-\frac{\beta}{8k}$,~~$ \err\left(h_j;(\cD_\cX, \whh)\right) \leq \frac{\alpha}{12}.$

From the above and the fact that the queries in $\cQ$ are i.i.d. from $\cD_\cX$, we invoke the same counting argument in the proof of \ifthenelse{\alt=1}{\citep[Theorem 3.2][]{bassily2018model}}{\cite[Theorem 3.2]{bassily2018model}} to show that w.p. $\geq 1-\frac{\beta}{4}$, the output labels of $\cA_\samp$ satisfy:
\begin{align}\label{eq:MisClrate}
    \frac{1}{m} \sum_{i=1}^{m} \ind\left(y_i^\prv \neq \whh(\tx_i)\right) &\leq \frac{\alpha}{4}.
\end{align}
Let $h_{S'}^\erm$ denote the ERM hypothesis with respect to the dataset $S'$ constructed in Step~\ref{Stp:smplS} of $\cA_{\Agprv}$. Note that Lemma~\ref{lem:RelabelAcc} implies that w.p. $\geq 1-\beta/4$, ~$\err(\whh; \cD)-\err\left(h_{S'}^\erm; \cD\right)\leq \alpha/4.$

Since the queries and their true labels $\big((\tx_1, \ty_1), \ldots, (\tx_m, \ty_m)\big)$ are drawn i.i.d. from $\cD$, then by Chernoff's bound and the fact that $m\geq 8\frac{\log(1/\beta)}{\alpha}$, we get that w.p. $\geq 1-\frac{\beta}{2},$ 
\begin{align}
\frac{1}{m} \sum_{i=1}^{m}\ind\left(\whh(\tx_i)\neq\ty_i\right) - \frac{1}{m}\sum_{i=1}^{m}\ind\left(h_{S'}^\erm(\tx_i)\neq\ty_i\right) &\leq \frac{\alpha}{2}.\label{eq:misclass-whh-herm}
\end{align}
Moreover, from the bound on $n'$ and using a basic fact from learning theory, w.p. $\geq 1-\beta/8$, the ERM hypothesis $h_{S'}^\erm$ satisfies: ~$\err\left(h_{S'}^\erm; \cD\right)\leq \alpha/8 + \gamma,$ where $\gamma=\min\limits_{h\in\cH}\err(h; \cD)$.
Again, since $\big((\tx_1, \ty_1), \ldots, (\tx_m, \ty_m)\big)$ are i.i.d. from $\cD$, then by Chernoff's bound and the fact that $m\geq 8\frac{\log(1/\beta)}{\alpha}$, w.p. $\geq 1-\beta/4$, we have
\begin{align}
    \frac{1}{m} \sum_{i=1}^{m} \ind\left(h_{S'}^\erm(\tx_i) \neq \ty_i\right)&\leq \frac{\alpha}{4} +\gamma. \label{eq:misclass-herm}
\end{align}
Now, using (\ref{eq:MisClrate}), (\ref{eq:misclass-whh-herm}), and (\ref{eq:misclass-herm}) together with a simple application of the triangle inequality and the union bound, we conclude that w.p. $\geq 1-\beta$,  ~~~ $\frac{1}{m}\sum_{j=1}^{m} {\ind\left(y_j^\prv \neq \ty_j\right)} \leq \alpha + \gamma.$
\end{proof}}




\ifthenelse{\alt=0}{
\begin{proof}[Proof of Claim~\ref{clm:disRate}]~
	For $S_u \sim \cD_{\cX}^{n_o}$, define the event
	$$ {\sf Bad} = \{\exists h_1,h_2 \in \cH: |{\dis}(h_1,h_2;\cD_\cX) - \widehat{\dis}(h_1,h_2;S_u) | > \alpha \}$$
	We will show that $\pr{S_u \sim \cD_\cX^{n_o}}{{\sf Bad}} \leq 2 \left(\frac{en_o}{d}\right)^{2d} \exp{(-n_o\alpha^2/8)}$. Hence, by using a standard manipulation, one can easily show that the right-hand side is bounded by $\beta'$ when $n_o$ is as given in the statement of the claim.
	
	\noindent Let $\cH_\Delta$ be a hypothesis class defined as $\cH_\Delta \triangleq \{h_1 \Delta h_2 : h_1,h_2 \in \cH\}$, where $h_1 \Delta h_2:\cX \rightarrow \{0,1\}$ is defined as: ~~ $\forall x \in \cX,~~ h_1 \Delta h_2(x) \triangleq \ind(h_1(x) \neq h_2(x)).$
	
	Let $\cGTri$ denote the growth function of $\cH_\Delta$; i.e. for any number $t$, 
	~ $\cGTri(t) \triangleq \max_{V:|V|=t} \left| \prod_{\cH_\Delta}(V)\right|,$
	where $\prod_{\cH_\Delta}(V)$ is the set of all dichotomies that can be generated by $\cH_\Delta$ on a set $V$ of size $t$. Now for any set $V$ of size $t$, every dichotomy in $\prod_{\cH_\Delta}(V)$ is determined by a pair of dichotomies in $\prod_{\cH}(V)$, and thus we get $|\prod_{\cH_\Delta}(V)| \leq |\prod_{\cH}(V)|^2$. Hence, by Sauer's Lemma $\cGTri(t) \leq \cG_\cH(t) \leq \left(\frac{et}{d}\right)^{2d}$. 
	Let $h_0$ be the all-zero hypothesis. Note that $h_0 \in \cH_\Delta$. Now, using a standard $\vc$-based uniform convergence argument we have,
	\begin{align*}
	&\pr{S_u \sim \cD_\cX^{n_o}}{\exists h_1,h_2 \in \cH: |{\dis}(h_1,h_2;\cD_\cX) - \widehat{\dis}(h_1,h_2;S_u) | > \alpha } \\
	&\leq \pr{S_u \sim \cD_\cX^{n_o}}{\exists h \in \cH_\Delta: |{\dis}(h,h_0;\cD_\cX) - \widehat{\dis}(h,h_0);S_u) | > \alpha } \\
	&\leq 2\cGTri \exp{(-n_o\alpha^2/8)} 
	\leq 2 \left(\frac{en_o}{d}\right)^{2d} \exp{(-n_o\alpha^2/8)}
	\end{align*}
	Note that the first inequality in the third line is non-trivial, and is used unanimously in VC-based uniform convergence bounds \ifthenelse{\alt=1}{\cite[see e.g.][]{shalev2014understanding}}{(see e.g., \cite{shalev2014understanding})}. 
\end{proof}
}
{\begin{proof}{\bf of Claim~\ref{clm:disRate}}~
	For $S_u \sim \cD_{\cX}^{n_o}$, define the event
	$$ {\sf Bad} = \{\exists h_1,h_2 \in \cH: |{\dis}(h_1,h_2;\cD_\cX) - \widehat{\dis}(h_1,h_2;S_u) | > \alpha \}$$
	We will show that $\pr{S_u \sim \cD_\cX^{n_o}}{{\sf Bad}} \leq 2 \left(\frac{en_o}{d}\right)^{2d} \exp{(-n_o\alpha^2/8)}$. Hence, by using a standard manipulation, one can easily show that the right-hand side is bounded by $\beta'$ when $n_o$ is as given in the statement of the claim.
	
	\noindent Let $\cH_\Delta$ be a hypothesis class defined as $\cH_\Delta \triangleq \{h_1 \Delta h_2 : h_1,h_2 \in \cH\}$, where $h_1 \Delta h_2:\cX \rightarrow \{0,1\}$ is defined as: ~~ $\forall x \in \cX,~~ h_1 \Delta h_2(x) \triangleq \ind(h_1(x) \neq h_2(x)).$
	
	Let $\cGTri$ denote the growth function of $\cH_\Delta$; i.e. for any number $t$, 
	~ $\cGTri(t) \triangleq \max_{V:|V|=t} \left| \prod_{\cH_\Delta}(V)\right|,$
	where $\prod_{\cH_\Delta}(V)$ is the set of all dichotomies that can be generated by $\cH_\Delta$ on a set $V$ of size $t$. Now for any set $V$ of size $t$, every dichotomy in $\prod_{\cH_\Delta}(V)$ is determined by a pair of dichotomies in $\prod_{\cH}(V)$, and thus we get $|\prod_{\cH_\Delta}(V)| \leq |\prod_{\cH}(V)|^2$. Hence, by Sauer's Lemma $\cGTri(t) \leq \cG_\cH(t) \leq \left(\frac{et}{d}\right)^{2d}$. 
	Let $h_0$ be the all-zero hypothesis. Note that $h_0 \in \cH_\Delta$. Now, using a standard $\vc$-based uniform convergence argument we have,
	\begin{align*}
	&\pr{S_u \sim \cD_\cX^{n_o}}{\exists h_1,h_2 \in \cH: |{\dis}(h_1,h_2;\cD_\cX) - \widehat{\dis}(h_1,h_2;S_u) | > \alpha } \\
	&\leq \pr{S_u \sim \cD_\cX^{n_o}}{\exists h \in \cH_\Delta: |{\dis}(h,h_0;\cD_\cX) - \widehat{\dis}(h,h_0);S_u) | > \alpha } \\
	&\leq 2\cGTri \exp{(-n_o\alpha^2/8)} 
	\leq 2 \left(\frac{en_o}{d}\right)^{2d} \exp{(-n_o\alpha^2/8)}
	\end{align*}
	Note that the first inequality in the third line is non-trivial, and is used unanimously in VC-based uniform convergence bounds \ifthenelse{\alt=1}{\cite[see e.g.][]{shalev2014understanding}}{(see e.g., \cite{shalev2014understanding})}. 
\end{proof}}

\section{Privately Answering Any Number of Classification Queries}\label{sec:privLearn}
In this section, we describe an universal $\pcqr$ algorithm that can answer \emph{any} number of queries with private sample size that is independent of the number of queries.
The main idea is that after answering a number of queries $\approx \frac{\vc(\cH)}{\alpha}$, we can use the feature-vectors defining those queries as an auxiliary ``public'' dataset. Recall that as defined earlier in our problem statement, the set of queries themselves do not entail any privacy constraints. We can then invoke the framework of \textit{semi-private learning} introduced \ifthenelse{\alt=1}{by}{in} \cite{beimel2013private}, where such auxiliary public dataset can be exploited to finally generate a classifier that is safe to publish. In particular, a semi-private learner takes as input two types of datasets: a private labeled dataset, and another auxiliary public dataset. The algorithm needs to satisfy differential privacy only with respect to the private dataset. 
\ifthenelse{\alt=1}{}{Alon et al.}\cite{ABM19} describe a construction of a semi-private learner (referred to as $\cA_\sspp$), and show that it suffices to have a public unlabeled dataset of size $\approx \frac{\vc(\cH)}{\alpha}$ to privately learn any hypothesis class $\cH$ with excess error $\alpha$ in the \emph{agnostic} setting. In particular, for any distribution $\cD$ over $\cX \times \cY$, given a set of feature-vectors of size $\approx \frac{\vc(\cH)}{\alpha}$ drawn i.i.d. from $\cD_\cX$, and a private labeled training set of $\approx \frac{\vc(\cH)}{\epsilon~\alpha^2}$ drawn i.i.d. from $\cD$, $\cA_\sspp$ outputs a classifier $h_\prv \in \cH$ such that $\err(h_\prv;\cD) - \min\limits_{h\in \cH}\err(h; \cD) \leq \alpha$ w.r.t $\cD$. Hence, $\cA_\sspp$ outputs a  classifier that can be used to answer any number of subsequent classification queries. For the sake of completeness, we give a formal description of $\cA_\sspp$ in Appendix~\ref{apndx:subsamp}.

Using this result, we can extend our construction in Section~\ref{sec:AgSecCl} to allow  for privately answering any number of classification queries using a private training set whose size is independent of the number of queries. In Algorithm~\ref{Alg:UnivSeqClass} below (denoted as $\cA_\uniprv$), we describe our universal $\pcqr$ algorithm.  
\begin{algorithm}[H]
 		\caption{$\cA_{\uniprv}$: Universal Private Classification-Query Release Algorithm}
 		\begin{algorithmic}[1]
 			\REQUIRE Private dataset: $S \in U^n$, number of queries: $m$, online sequence of classification queries: $\cQ=(\tx_1, \ldots ,\tx_m)$, hypothesis class: $\cH$, oracle access to a non-private PAC learner for $\cH$: $\cB_\pac$, ~privacy parameters $\epsilon,\delta >0$,~ accuracy parameter $\alpha$, and failure probability $\beta$.
 			\STATE $m_o \leftarrow 32\,\frac{d\log(1/\alpha) + \log(1/\beta)}{\alpha}, ~~m' \leftarrow \min(m_o,m)$ \label{stp:initialmo}
 			\STATE Output $(y^\prv_1, \ldots, y^\prv_{m'})  \leftarrow \cA_\Agprv\bigg(S, ~m', ~(\tx_1,\ldots,\tx_{m'}),~\cH,~ \cB_\pac, ~\epsilon, ~\delta, ~\alpha, ~\beta\bigg)$
 			\IF{$m'=m_o$}
 			\STATE $T_\pub \leftarrow (\tx_1, \ldots,\tx_{m_o})$
 			\STATE $h_\prv \leftarrow \cA_{\sspp}(S,T_\pub,\cH,\epsilon)$ 
	        \FOR{$j=m_o+1, \ldots, m$}
	            \STATE Output $y^\prv_j \leftarrow h_\prv(\tx_j)$
	       \ENDFOR
	       \ENDIF
 		\end{algorithmic}
 		\label{Alg:UnivSeqClass}
\end{algorithm}

 \noindent We finally formalize this observation in the following theorem.
 \begin{thm}\label{thm:univSam}
 Let $\cH$ be any hypothesis class with $\vc(\cH)=d$. For any $\epsilon,\delta,\alpha,\beta \in (0,1)$ and any $m < \infty$, $\cA_\uniprv$ is an $(\epsilon,\delta,\alpha,\beta,n,m)$-$\pcqr$ algorithm for $\cH$ with private sample size
  $$n = O\left(\frac{\left(d\,\log\left(\frac{1}{\alpha}\right) + \log\left(\frac{m_o}{\beta}\right)\right)\, \log^{3/2}\left(\frac{2}{\delta}\right)\, \log\left(\frac{m_o\alpha}{\min(\delta,\beta/2)}\right)}{\epsilon ~\alpha^2}\cdot\max \left(1,\sqrt{d}~\alpha\log^{1/2}\left(\frac{1}{\alpha}\right)\right)\right),$$
  where $m_o = O\left(\frac{d\log(1/\alpha) + \log(1/\beta)}{\alpha}\right)$ (as set in Step~\ref{stp:initialmo}). In particular, when $\alpha \leq \frac{1}{\sqrt{d}}$, it would suffice to have a private sample of size $n = \tilde{O}\left(\frac{d}{\epsilon ~\alpha^2}\right)$.
\end{thm}
 

\paragraph{Near-optimality of our sample complexity bound:} Note that without any privacy constraints, the sample complexity of this problem in the agnostic PAC setting is $\Theta\left(\left(\frac{\vc(\cH) + \log\left(1/\beta\right)}{\alpha^2}\right)\right)$. Note that this follows from the standard agnostic PAC learning bound and the fact that access to unlabeled data (the set of queries) does not improve the sample complexity (\cite[Theorem 15]{ben2008does}, unless one makes assumptions about the conditional distribution of the true label given the unlabeled domain point). Now, when $\alpha=O\left(\frac{1}{\sqrt{\vc(\cH)}}\right)$, and assuming $\epsilon = \Theta(1)$, our private sample complexity bound in Theorem~\ref{thm:univSam} nearly matches the non-private sample complexity. This shows that our bound is optimal (up to log factors) in that parameters regime\footnote{Note that the accuracy of $\cA_\uniprv$ is defined in terms of the average misclassification rate over the given set of queries rather the expected classification error, however, since the queries are i.i.d., it is easy to see that the two accuracy definitions are essentially equivalent (by applying a standard concentration argument).}. 

\begin{remark}
 It is worth mentioning that applying the same technique (the semi-private learner \ifthenelse{\alt=1}{of}{given in} \cite{ABM19}) to the construction of \ifthenelse{\alt=1}{\citep{bassily2018model}}{\cite{bassily2018model}} also yields a universal PCQR algorithm but with a worse sample complexity bound than ours. In particular, it is not hard to see that the resulting sample complexity bound based on the construction \ifthenelse{\alt=1}{by}{of} \cite{bassily2018model} is $\tilde{O}\left(\frac{\vc(\cH)}{\alpha^2}\cdot\max\left(1,\sqrt{\vc(\cH)}\right)\right),$ where $\tilde{O}$ hides $\polylog$ factors in $(1/\alpha, 1/\beta, 1/\delta)$. Our bound is tighter by a factor of $\approx \alpha$ when $\vc(\cH)=\omega(1)$. 
\end{remark}

\ifnum\alt=1
\acks{The authors would like to thank Uri Stemmer, Amos Beimel, and Kobbi Nissim for pointing us to their elegant LabelBoost procedure in \cite{BeimelNS14} which is the crux of the pre-processing step of our algorithm. We are also grateful to them for the several insightful discussions we had about this line of research. We also thank Vitaly Feldman for his helpful comments on the manuscript, and the anonymous reviewer for pointing out the fact that we can save a factor of $1/\epsilon$ ~(in an earlier version of the bounds) via a small tweak in the argument.
}

\vskip 0.2in
\bibliography{reference}
\else
\section*{Acknowledgements}
The authors would like to thank Uri Stemmer, Amos Beimel, and Kobbi Nissim for pointing us to their elegant LabelBoost procedure in \cite{BeimelNS14} which is the crux of the pre-processing step of our algorithm. We are also grateful to them for the several insightful discussions we had about this line of research. We also thank Vitaly Feldman for his helpful comments on the manuscript, and the anonymous reviewer for pointing out the fact that we can save a factor of $1/\epsilon$ ~(in an earlier version of the bounds) via a small tweak in the argument.

\bibliographystyle{alpha} 
\bibliography{reference}

\newcommand{\etalchar}[1]{$^{#1}$}
\begin{thebibliography}{DKM{\etalchar{+}}06}

\bibitem[ABM19]{ABM19}
Noga Alon, Raef Bassily, and Shay Moran.
\newblock Limits of private learning with access to public data.
\newblock {\em To appear in NeuRIPS 2019, also available at arXiv:1910.11519
  [cs.LG]}, 2019.

\bibitem[ALMM18]{almm19}
Noga Alon, Roi Livni, Maryanthe Malliaris, and Shay Moran.
\newblock Private pac learning implies finite littlestone dimension.
\newblock {\em arXiv preprint arXiv:1806.00949 (STOC 2019, in Press)}, 2018.

\bibitem[BDLP08]{ben2008does}
Shai Ben-David, Tyler Lu, and D{\'a}vid P{\'a}l.
\newblock Does unlabeled data provably help? worst-case analysis of the sample
  complexity of semi-supervised learning.
\newblock In {\em COLT}, pages 33--44, 2008.

\bibitem[BNS13]{beimel2013private}
Amos Beimel, Kobbi Nissim, and Uri Stemmer.
\newblock Private learning and sanitization: Pure vs. approximate differential
  privacy.
\newblock In {\em Approximation, Randomization, and Combinatorial Optimization.
  Algorithms and Techniques}, pages 363--378. Springer, 2013.

\bibitem[BNS15]{BeimelNS14}
Amos Beimel, Kobbi Nissim, and Uri Stemmer.
\newblock Learning privately with labeled and unlabeled examples.
\newblock {\em CoRR}, abs/1407.2662 (appeared at SODA 2015), 2015.

\bibitem[BNSV15]{bun2015differentially}
Mark Bun, Kobbi Nissim, Uri Stemmer, and Salil Vadhan.
\newblock Differentially private release and learning of threshold functions.
\newblock In {\em Foundations of Computer Science (FOCS), 2015 IEEE 56th Annual
  Symposium on}, pages 634--649. IEEE, 2015.

\bibitem[BTT18]{bassily2018model}
Raef Bassily, Abhradeep Thakurta, and Om~Thakkar.
\newblock Model-agnostic private learning.
\newblock In {\em Advances in Neural Information Processing Systems 31}, pages
  7102--7112. Curran Associates, Inc., 2018.

\bibitem[DF18]{dwork2018privacy}
Cynthia Dwork and Vitaly Feldman.
\newblock Privacy-preserving prediction.
\newblock {\em arXiv preprint arXiv:1803.10266}, 2018.

\bibitem[DF19]{Dagan-Feldman2019}
Yuval Dagan and Vitaly Feldman.
\newblock Pac learning with stable and private predictions.
\newblock {\em arXiv:1911.10541 [cs.LG]}, 2019.

\bibitem[DKM{\etalchar{+}}06]{DKMMN06}
Cynthia Dwork, Krishnaram Kenthapadi, Frank McSherry, Ilya Mironov, and Moni
  Naor.
\newblock Our data, ourselves: Privacy via distributed noise generation.
\newblock In {\em EUROCRYPT}, 2006.

\bibitem[DMNS06]{DMNS06}
Cynthia Dwork, Frank McSherry, Kobbi Nissim, and Adam Smith.
\newblock Calibrating noise to sensitivity in private data analysis.
\newblock In {\em Theory of Cryptography Conference}, pages 265--284. Springer,
  2006.

\bibitem[DR14]{DR14}
Cynthia Dwork and Aaron Roth.
\newblock The algorithmic foundations of differential privacy.
\newblock {\em Foundations and Trends in Theoretical Computer Science},
  9(3-4):211--407, 2014.

\bibitem[HCB16]{hamm2016learning}
Jihun Hamm, Yingjun Cao, and Mikhail Belkin.
\newblock Learning privately from multiparty data.
\newblock In {\em International Conference on Machine Learning}, pages
  555--563, 2016.

\bibitem[KLN{\etalchar{+}}08]{KLNRS08}
Shiva~Prasad Kasiviswanathan, Homin~K. Lee, Kobbi Nissim, Sofya Raskhodnikova,
  and Adam Smith.
\newblock What can we learn privately?
\newblock In {\em FOCS}, pages 531--540. IEEE Computer Society, 2008.

\bibitem[LQS12]{li2012sampling}
Ninghui Li, Wahbeh Qardaji, and Dong Su.
\newblock On sampling, anonymization, and differential privacy or,
  k-anonymization meets differential privacy.
\newblock In {\em Proceedings of the 7th ACM Symposium on Information, Computer
  and Communications Security}, pages 32--33. ACM, 2012.

\bibitem[MT07]{MT07}
Frank McSherry and Kunal Talwar.
\newblock Mechanism design via differential privacy.
\newblock In {\em FOCS}, 2007.

\bibitem[NRS07]{NRS07}
Kobbi Nissim, Sofya Raskhodnikova, and Adam Smith.
\newblock Smooth sensitivity and sampling in private data analysis.
\newblock In {\em STOC}, 2007.

\bibitem[PAE{\etalchar{+}}17]{papernot2017semi}
Nicolas Papernot, Mart{\i}n Abadi, \'Ulfar Erlingsson, Ian Goodfellow, and
  Kunal Talwar.
\newblock Semi-supervised knowledge transfer for deep learning from private
  training data.
\newblock {\em stat}, 1050, 2017.

\bibitem[PSM{\etalchar{+}}18]{papernot2018scalable}
Nicolas Papernot, Shuang Song, Ilya Mironov, Ananth Raghunathan, Kunal Talwar,
  and {\'U}lfar Erlingsson.
\newblock Scalable private learning with pate.
\newblock {\em arXiv preprint arXiv:1802.08908}, 2018.

\bibitem[Sau72]{sauer1972density}
Norbert Sauer.
\newblock On the density of families of sets.
\newblock {\em Journal of Combinatorial Theory, Series A}, 13(1):145--147,
  1972.

\bibitem[SSBD14]{shalev2014understanding}
Shai Shalev-Shwartz and Shai Ben-David.
\newblock {\em Understanding machine learning: From theory to algorithms}.
\newblock Cambridge university press, 2014.

\bibitem[ST13]{ST13}
Adam Smith and Abhradeep Thakurta.
\newblock Differentially private feature selection via stability arguments, and
  the robustness of the lasso.
\newblock In {\em COLT}, 2013.

\end{thebibliography}
\fi
\newpage
\appendix
\section{Proof of Lemma \ref{lem:RelabelAcc}}\label{apndx:prfRelb}
		

\ifthenelse{\alt=0}{\begin{proof}[Proof of Lemma~\ref{lem:RelabelAcc}]

  Note that the score function for the exponential mechanism is $-\herr(h; S')$ whose global sensitivity is $1/n'$. 
  Now, by using the standard accuracy guarantees of exponential mechanism of \ifthenelse{\alt=1}{\citep{MT07}}{\cite{MT07}} (and the fact that its instantiated here with privacy parameter $=1$), w.p. $\geq 1-\beta/3$ we have
  $$\herr\left(\whh; S'\right)- \herr\left(h_{S'}^\erm; S'\right)\leq  \frac{2}{n'}\left(\log\left(|\widetilde{H}|\right) + \log(\frac{3}{\beta})\right).$$
 Using the value of $n'$ given in the lemma statement, together with Sauer's Lemma \ifthenelse{\alt=1}{\citep{sauer1972density}}{\cite{sauer1972density}} to bound the size of $\tilde{H}$, it follows that:
 \begin{align}
\herr\left(\whh; S'\right)- \herr\left(h_{S'}^\erm; S'\right)  &\leq \frac{2}{n'}\left(d \log\left(\frac{en'}{d}\right)+ \log\left(\frac{3}{\beta}\right)\right)\nonumber\\
&\leq \frac{80\alpha^2~\left(d\log(1/\alpha) + \log(3/\beta)\right)}{256(\,d + \log(3/\beta))}\nonumber \\
&\leq \alpha/3.\label{ineq:emp-errors}
\end{align}
Given the bound on $n'$ and the fact that $S'\sim\cD^{n'}$, by a standard uniform convergence argument from learning theory \ifthenelse{\alt=1}{\citep{shalev2014understanding}}{\cite{shalev2014understanding}}, we have the following generalization error bounds. With probability $\geq 1-2\beta/3,$ we have:
\begin{align}
    \lvert\err(\whh; \cD)-\herr(\whh; S')\rvert&\leq \alpha/3,\label{ineq:gen-err-whh}\\
    \lvert\err\left(h_{S'}^\erm; \cD\right)-\herr\left(h_{S'}^\erm; S'\right)\rvert&\leq \alpha/3\label{ineq:gen-err-herm}
\end{align}
Putting (\ref{ineq:emp-errors})-(\ref{ineq:gen-err-herm}) together, we conclude that w.p. $\geq 1-\beta,$ we have $\err\left(\whh; \cD\right)- \err\left(h_{S'}^\erm; \cD\right)\leq \alpha.$ This completes the proof.
\end{proof}}
{\begin{proof}{\bf of Lemma~\ref{lem:RelabelAcc}}

  Note that the score function for the exponential mechanism is $-\herr(h; S')$ whose global sensitivity is $1/n'$. 
  Now, by using the standard accuracy guarantees of exponential mechanism of \ifthenelse{\alt=1}{\citep{MT07}}{\cite{MT07}} (and the fact that its instantiated here with privacy parameter $=1$), w.p. $\geq 1-\beta/3$ we have
  $$\herr\left(\whh; S'\right)- \herr\left(h_{S'}^\erm; S'\right)\leq  \frac{2}{n'}\left(\log\left(|\widetilde{H}|\right) + \log(\frac{3}{\beta})\right).$$
 Using the value of $n'$ given in the lemma statement, together with Sauer's Lemma \ifthenelse{\alt=1}{\citep{sauer1972density}}{\cite{sauer1972density}} to bound the size of $\tilde{H}$, it follows that:
 \begin{align}
\herr\left(\whh; S'\right)- \herr\left(h_{S'}^\erm; S'\right)  &\leq \frac{2}{n'}\left(d \log\left(\frac{en'}{d}\right)+ \log\left(\frac{3}{\beta}\right)\right)\nonumber\\
&\leq \frac{80\alpha^2~\left(d\log(1/\alpha) + \log(3/\beta)\right)}{256(\,d + \log(3/\beta))}\nonumber \\
&\leq \alpha/3.\label{ineq:emp-errors}
\end{align}
Given the bound on $n'$ and the fact that $S'\sim\cD^{n'}$, by a standard uniform convergence argument from learning theory \ifthenelse{\alt=1}{\citep{shalev2014understanding}}{\cite{shalev2014understanding}}, we have the following generalization error bounds. With probability $\geq 1-2\beta/3,$ we have:
\begin{align}
    \lvert\err(\whh; \cD)-\herr(\whh; S')\rvert&\leq \alpha/3,\label{ineq:gen-err-whh}\\
    \lvert\err\left(h_{S'}^\erm; \cD\right)-\herr\left(h_{S'}^\erm; S'\right)\rvert&\leq \alpha/3\label{ineq:gen-err-herm}
\end{align}
Putting (\ref{ineq:emp-errors})-(\ref{ineq:gen-err-herm}) together, we conclude that w.p. $\geq 1-\beta,$ we have $\err\left(\whh; \cD\right)- \err\left(h_{S'}^\erm; \cD\right)\leq \alpha.$ This completes the proof.
\end{proof}}

\section{Constructions from previous works}\label{apndx:subsamp}

\subsection{Description of Algorithm $\cA_\samp$}

For completeness, here we briefly describe the algorithm $\cA_\samp$ (Algorithm~\ref{Alg:binClas} below) \ifthenelse{\alt=1}{due to}{in} \cite{bassily2018model}. The input to $\cA_\samp$ is a private labeled dataset $S=\{(x_1,y_1),\ldots,(x_n,y_n)\}$, an online sequence of classification queries $Q = (\tx_1, \ldots, \tx_m)$, and a generic non-private PAC learner $\cB$ for a hypothesis class $\cH$. The algorithm outputs a sequence of private labels $(y_1^\prv,\dots,y_m^\prv)$. The key idea in $\A_\samp$ is as follows: first, it arbitrarily splits $S$ into $k$ equal-sized sub-samples $S_1,\dots,S_k$ for appropriately chosen $k$. Each of those sub-samples is used to train $\cB$. Hence, we obtain an ensemble of $k$ classifiers $h_{S_1},\cdots,h_{S_k}$. Next for each input query $\tx_i \in Q$, the votes $(h_{S_1}(\tx_i), \ldots, h_{S_k}(\tx_i))$ are computed. It then applies the distance-to-instability test \ifthenelse{\alt=1}{\citep{ST13}}{\cite{ST13}} on the difference between the largest count of votes and the second largest count. 
If the majority vote is sufficiently stable, $\cA_{\samp}$ returns the majority vote as the predicted label for $\tx_i$; otherwise, it returns a random label. The sparse-vector framework is employed to efficiently manage the privacy budget over the $m$ queries. In particular, by employing the sparse-vector technique, the privacy budget of $\cA_{\samp}$ is only consumed by those queries where the majority vote is not stable. Algorithm $\cA_{\samp}$ takes an input cut-off parameter $T$, which represents a bound on the total number of ``unstable queries'' the algorithm can answer before it halts in order to ensure $(\epsilon,\delta)$-differential privacy. 


 	\begin{algorithm}[H]
 		\caption{$\A_\samp$ \cite{bassily2018model}: Private Classification via subsample-aggregate and sparse-vector}
 		\begin{algorithmic}[1]
 			\REQUIRE Private dataset: $S$, ~ upper bound on the number of queries: $m$, ~online sequence of classification queries: $\cQ=\{\tx_1,\ldots,\tx_m\}$, ~ hypothesis class $\cH$, ~oracle access to a PAC learner of $\cH$: $\cB_\pac$, ~unstable query cutoff: $T$, ~privacy parameters: \mbox{$\epsilon,\delta >0$}, ~failure probability: $\beta$.
 			\STATE $c \leftarrow 0, ~\lambda \leftarrow \frac{\sqrt{32T\log(2/\delta)}}{\epsilon}$ and $k \leftarrow 34\sqrt{2}\lambda\cdot\log\left(4mT/\min\left(\delta, \beta/2\right)\right)$ \label{Stp:kAssn}
 			\STATE $w \leftarrow 2\lambda\cdot \log(2m/\delta),~ \hw \leftarrow w + \sf Lap(\lambda)$  ~~\COMMENT{$\sf Lap(b)$ denotes the Laplace distribution with scale $b$} 
			\STATE Split $S$ into $k$ non-overlapping sub-samples $S_1,\cdots, S_k$. 
 			\FOR{$j \in [k]$} \STATE $h_{S_j} \leftarrow \cB_\pac(S_{j})$
 			\ENDFOR
			\FOR{$i\in[m]$ and $c \leq T$}
 			\STATE $\cF_i\leftarrow\{h_{S_1}(x_i),\cdots,h_{S_k}(x_i)\}$ \COMMENT{For every $y \in \{0,1\}$, let ${\sf ct}(y)=\#$ times $y$ appears in $\cF_i$.} 
			\STATE $\widehat{q}_{x_i} \leftarrow \arg\max\limits_{y \in \{0,1\}}\left[{\sf ct}(y)\right]$, ~~$\dist_{\widehat{y}_{x_i}} \leftarrow$ largest ${\sf ct}(y)$ - second largest ${\sf ct}(y)$  
			\STATE $y_i^\prv \leftarrow \cA_\stab(S,\widehat{q}_{x_i},\dist_{\widehat{y}_{x_i}},\hw,\frac{1}{2\lambda})$ \COMMENT {Stability test for $\widehat{q}_{x_i}$, given by Algorithm~\ref{Alg:kStab} below.}
			\STATE {\bf if} $y_i^\prv=\bot$, {\bf then} $c \leftarrow c + 1$, ~$\hw \leftarrow w + \sf Lap(\lambda)$
			\STATE Output $y_i^\prv$
			\ENDFOR 
		\end{algorithmic}
		\label{Alg:binClas}
	\end{algorithm}
	
\begin{algorithm}
	\caption{$\A_{\stab}$ \cite{ST13}: Private estimator for $f$ via distance to instability}
	\begin{algorithmic}[1]
		\REQUIRE Dataset: $S$, function: $f:U^n\to\R$, distance to instability: $\dist_f:U^n\to\mathbb{R}$, threshold: $\thr$, privacy parameter: $\epsilon>0$
		\STATE $\wds\leftarrow\dist_f(S)+{\sf Lap}\left(1/\epsilon\right)$
		\STATE {\bf if} $\wds > \thr$, {\bf then} output $f(S)$, {\bf else} output $\bot$ 
	\end{algorithmic}
	\label{Alg:kStab}
\end{algorithm}

\subsection{Description of Algorithm $\cA_\sspp$}

In Section~\ref{sec:privLearn}, we use a semi-supervised semi-private learner construction from \ifthenelse{\alt=1}{\citep{ABM19}}{\cite{ABM19}} (referred to as $\cA_\sspp$) to give a construction for a universal $\pcqr$ algorithm that can answer any number of classification queries (Algorithm~\ref{Alg:UnivSeqClass}). 
For completeness, we describe the construction of this semi-private learner $\cA_\sspp$ in Algorithm~\ref{Alg:exp-net} below\footnote{A similar construction of the semi-private learner $\A_\sspp$ has also appeared in the earlier work by \ifthenelse{\alt=1}{}{Beimel et al. in}\cite{beimel2013private}.}. 
Algorithm~\ref{Alg:exp-net} takes as input two datasets: a private dataset $S$ of size $n$, and an unlabeled public dataset $\Tpub$ of size $m_o$, and outputs a hypothesis $h_\prv : \cX \rightarrow \{0,1\}$. The main idea of the construction in \ifthenelse{\alt=1}{\citep{ABM19}}{\cite{ABM19}} is that the public unlabeled dataset can be used to create a finite $\alpha$-cover for $\cH$ (see Definition~\ref{def:acover} below), and hence, reducing the task of privately learning $\cH$ to the task of learning a finite sub-class of $\cH$ (the $\alpha$-cover).

\begin{defn}[$\alpha$-cover for a hypothesis class]\label{def:acover}
A family of hypotheses $\wtH$ is said to form an $alpha$-cover for a hypothesis class $\cH \subseteq \{0,1\}^{\cX} $ with respect to distribution $\cD_\cX$ if for every $h \in \cH$ there exists a $\tlh \in \wtH$ such that $\ex{x \sim \cD_\cX}{\ind(h(x) \neq \tlh(x)} \leq \alpha$. 
\end{defn}
\begin{algorithm}[H]
	\caption{$\cA_\sspp$ \cite{ABM19}: Semi-Supervised Semi-Private Agnostic Learner}
	\begin{algorithmic}[1]
		\REQUIRE Private labeled dataset: $S \in U^n$, a public unlabeled dataset: $\Tpub=(\tx_1,\cdots, \tx_{m_o})\in \cX^{m_o}$, ~~\mbox{a hypothesis} class $\cH\subset \{0, 1\}^{\cX}$, and a privacy parameter $\epsilon >0$.
		\STATE Let $\tT=\{\hx_1, \ldots,\hx_{\hatm}\}$ be the set of points $x\in\cX$ appearing at least once in $\Tpub$.
		\STATE Let $\Pi_{\cH}(\tT)=\left\{\left(h(\hx_1), \ldots, h(\hx_{\hatm})\right):~h\in\cH\right\}.$
		\STATE Initialize $\tH_{\Tpub}=\emptyset$.\label{step:init}
		\FOR{each $\bc=(c_1, \ldots, c_{\hatm})\in\Pi_{\cH}(\tT)$:}
		\STATE Add to $\tH_{\Tpub}$ arbitrary $h\in\cH$ that satisfies $h(\hx_j)=c_j$ for every $j=1,\ldots, \hatm$.  \label{step:rep-hyp}
		\ENDFOR
		\STATE Use the exponential mechanism with inputs $S,~\tH_{\Tpub}, \eps,$~ and score function $q(S, h)\triangleq -\herr(h; S)$ to select $\hpv\in \tH_{\Tpub}$. \label{step:exp-mech} 
		\RETURN $\hpv.$
		
	\end{algorithmic}
	\label{Alg:exp-net}
\end{algorithm}

\end{document}